\documentclass[sigconf]{acmart}
\AtBeginDocument{%
  \providecommand\BibTeX{{%
    \normalfont B\kern-0.5em{\scshape i\kern-0.25em b}\kern-0.8em\TeX}}}

\usepackage{algorithm,algorithmicx,algpseudocode}
\usepackage{multirow}
\usepackage{graphicx}
\usepackage{textcomp}
\usepackage{xcolor}
\usepackage{booktabs}
\usepackage{hyperref}
\usepackage{balance}
\usepackage{bm}

\usepackage{amsthm}
\newtheorem{ttheorem}{Theorem}

\usepackage{array}

\usepackage{xspace}
\newcommand{\framework}{{DetectorGuard}\xspace}
\newcommand{\base}{{Base Detector}\xspace}
\newcommand{\defense}{{Objectness Predictor}\xspace}
\newcommand{\matcher}{{Objectness Explainer}\xspace}
\algnewcommand{\LeftCommenta}[1]{\Statex \hspace{1.3em} \(\triangleright\) #1}
\algnewcommand{\LeftCommentb}[1]{\Statex \hspace{2.8em} \(\triangleright\) #1}

\copyrightyear{2021}
\acmYear{2021}
\setcopyright{rightsretained}
\acmConference[CCS '21]{Proceedings of the 2021 ACM SIGSAC Conference on Computer and Communications Security}{November 15--19, 2021}{Virtual Event, Republic of Korea}
\acmBooktitle{Proceedings of the 2021 ACM SIGSAC Conference on Computer and Communications Security (CCS '21), November 15--19, 2021, Virtual Event, Republic of Korea}\acmDOI{10.1145/3460120.3484757}
\acmISBN{978-1-4503-8454-4/21/11}

\begin{document}
\fancyhead{}

\title{\framework: Provably Securing Object Detectors\\ against Localized Patch Hiding Attacks}

\author{Chong Xiang}
\affiliation{%
   \institution{Princeton University}
   \city{Princeton}
   \state{NJ}
   \country{USA}}
\email{cxiang@princeton.edu}
\author{Prateek Mittal}
\affiliation{%
   \institution{Princeton University}
   \city{Princeton}
   \state{NJ}
   \country{USA}}
\email{pmittal@princeton.edu}

\begin{abstract}
State-of-the-art object detectors are vulnerable to localized patch hiding attacks, where an adversary introduces a small adversarial patch to make detectors miss the detection of salient objects. The patch attacker can carry out a physical-world attack by printing and attaching an adversarial patch to the victim object; thus, it imposes a challenge for the safe deployment of object detectors. In this paper, we propose \framework as the first general framework for building provably robust object detectors against localized patch hiding attacks. \framework is inspired by recent advancements in robust image classification research; we ask: \textit{can we adapt robust image classifiers for robust object detection?} Unfortunately, due to their task difference, an object detector naively adapted from a robust image classifier 1) may not necessarily be robust in the adversarial setting or 2) even maintain decent performance in the clean setting. To address these two issues and build a high-performance robust object detector, we propose an \textit{objectness explaining strategy}: we adapt a robust image classifier to predict objectness (i.e., the probability of an object being present) for every image location and then explain each objectness using the bounding boxes predicted by a conventional object detector. If all objectness is well explained, we output the predictions made by the conventional object detector; otherwise, we issue an attack alert. Notably, our objectness explaining strategy enables provable robustness for ``free": 1) in the adversarial setting, we formally prove the end-to-end robustness of \framework on certified objects, i.e., it either detects the object or triggers an alert, against \textit{any patch hiding attacker} within our threat model; 2) in the clean setting, we have almost the same performance as state-of-the-art object detectors. Our evaluation on the PASCAL VOC, MS COCO, and KITTI datasets further demonstrates that \framework achieves the first provable robustness against localized patch hiding attacks at a negligible cost (<1\%) of clean performance.

\end{abstract}



\maketitle

\section{Introduction}

Localized adversarial patch attacks can induce mispredictions in Machine Learning (ML) systems and have gained significant attention over the past few years~\cite{brown2017adversarial,karmon2018lavan,thys2019fooling,wu2019making,xu2020adversarial}. A patch attacker constrains all adversarial perturbations within a small region so that they can carry out a physical world attack by printing and attaching the adversarial patch to the victim object.\footnote{The patch attack significantly differs from classic $L_p$-norm-bounded adversarial examples~\cite{szegedy2013intriguing,goodfellow2014explaining,carlini2017towards} that require global perturbations and thus are difficult to realize in the physical world.} To counter the threat of patch attacks on real-world ML systems, the security community has been actively seeking defense mechanisms~\cite{hayes2018visible,lin2017focal,zhang2020clipped,chiang2020certified,levine2020randomized,xiang2020patchguard,mccoyd2020minority}. However, most existing defenses are restricted to the image classification domain. In this paper, we aim to secure \textit{object detectors}, which are used in critical applications like autonomous driving, video surveillance, and identity verification~\cite{vahab2019applications}.

We focus on the threat of \textit{localized patch hiding attacks} against object detectors: an attacker uses a localized patch for physical world attacks that cause the object detector to fail to detect victim objects. Lee et al.~\cite{lee2019physical} show that a physical patch far away from the objects can successfully ``hide" victim objects. Wu et al.~\cite{wu2019making} and Xu et al.~\cite{xu2020adversarial} have succeeded in evading object detection via wearing a T-shirt printed with adversarial perturbations. The patch hiding attack can cause serious consequences in scenarios like an autonomous vehicle missing a pedestrian (or an upcoming car).

Unfortunately, securing object detectors is extremely challenging due to the complexity of the detection task. A single image can contain multiple objects, and thus an object detector needs to output a list of object bounding box coordinates and class labels. To the best of our knowledge, there is only one prior work~\cite{saha2020role} discussing defenses for YOLOv2~\cite{redmon2017yolo9000} detectors against patch attacks, in contrast to numerous new patch attacks being proposed~\cite{eykholt2018physical,chen2018shapeshifter,zhao2019seeing,thys2019fooling,wu2019making,xu2020adversarial}. Furthermore, this only defense~\cite{saha2020role} is restricted to the setting of a non-adaptive adversarial patch at the image corner and does not have any security guarantee (i.e., only heuristics-based). To overcome these weaknesses, we propose a defense framework named \framework that can achieve \textit{provable robustness} against any patch hiding attack within our threat model.

\begin{figure*}[t]
    \centering
    \includegraphics[width=\linewidth]{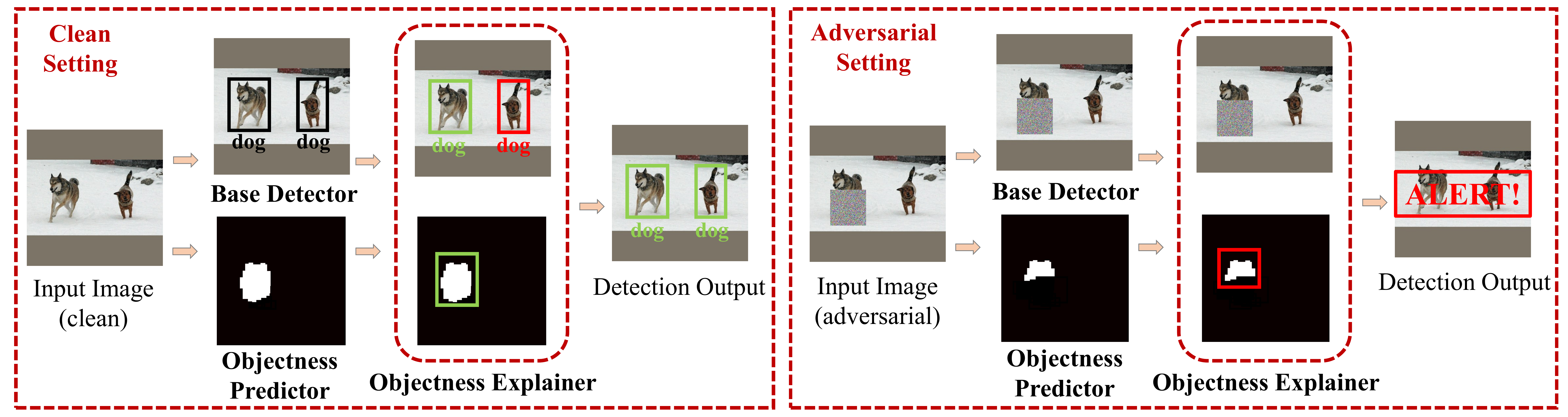}
    \caption{\framework Overview. \textmd{\textit{\base} is a conventional detector that typically predicts precise bounding boxes on clean images. \textit{\defense} aims to robustly predict an objectness map. \textit{\matcher} uses the predicted bounding boxes to explain/match the predicted objectness and determines the final output. \textit{In the clean setting (left figure)}, the dog on the left is detected by both \base and \defense. The objectness is well explained/matched by the green bounding box, and \framework outputs the bounding box predicted by \base. In the meantime, due to the imperfection of the robust classifier, the dog on the right is missed by \defense. \matcher will consider this as a \textit{benign mismatch}, and \framework will also output the predicted bounding box from \base. \textit{In the adversarial setting (right figure)}, a patch makes \base fail to detect any object while \defense still robustly outputs high objectness. \matcher detects a cluster of \textit{unexplained objectness} and triggers an attack alert. Our defense can achieve provable robustness for certified objects while maintaining a clean performance comparable to state-of-the-art object detectors.} }
    \label{fig-overview}
\end{figure*}

\textbf{Insight \& Challenge: a bridge between robust image classification and robust object detection.} To start with, we observe that the robust image classification research is making significant advancement~\cite{chiang2020certified,zhang2020clipped,levine2020randomized,mccoyd2020minority,xiang2020patchguard,metzen2021efficient} while object
detectors remain vulnerable to attacks. This sharp contrast motivates us to ask: \textit{can we take advantage of robust image classification research and adapt robust image classifiers for robust object detection?} Unfortunately, there is a huge gap between these two tasks: an image classifier outputs a single-label prediction while an object detector predicts a list of bounding boxes and class labels. This gap brings two major challenges for such an adaptation. \textit{Challenge 1: Lack of End-to-end Provable Robustness.} A robust image classifier only provides robustness for a single-label prediction while an object detector requires robustness for multiple object labels and bounding boxes in each image. Therefore, using a robust image classifier does not guarantee the end-to-end robustness of the object detector, and we need a new defense design and proof of robustness. \textit{Challenge 2: Amplified Cost of Clean Performance.} All existing provably robust image classifiers~\cite{chiang2020certified,zhang2020clipped,levine2020randomized,xiang2020patchguard,metzen2021efficient} attain robustness at a non-negligible cost of clean performance (e.g., state-of-the-art defense~\cite{xiang2020patchguard} incurs a >20\% clean accuracy drop on ImageNet~\cite{deng2009imagenet}). The imperfection of robust classifiers can be severely amplified during the adaptation towards the much more demanding object detection task. Therefore, we need to prevent our object detectors from being broken in the clean setting (even in the absence of an adversary). In \framework, we overcome these two challenges as discussed below. 

\textbf{Defense Design: an objectness explaining strategy.}  We provide our defense overview in Figure~\ref{fig-overview}. \framework has three modules: \base, \defense, and \matcher. \base can be any state-of-the-art object detector that can make \textit{accurate} predictions on clean images but is vulnerable to patch hiding attacks. \defense aims to predict a \textit{robust} objectness map, which indicates the probability of an object being present at different locations. We build \defense using adapted provably robust image classifiers together with carefully designed feature-space operation and error filtering mechanisms (Section~\ref{sec-obj-predictor}). Finally, \matcher uses each predicted bounding box from \base to \textit{explain/match} high objectness predicted by \defense (Section~\ref{sec-matcher}). 
If all objectness is well explained/matched, we output the prediction of \base; otherwise, we issue an attack alert. In the clean setting, we optimize the configuration of \defense towards the case where all objectness can be explained and then use \base for accurate final predictions (Figure~\ref{fig-overview} left). When a hiding attack occurs, \base could miss the object while \defense can still robustly output high objectness. This will lead to unexplained objectness and trigger an attack alert (Figure~\ref{fig-overview} right). Notably, we can show that our defense design successfully addresses \textit{Challenge 1 and 2}, as discussed next.

\textbf{End-to-end provable robustness for ``free".} First, our objectness explaining strategy enables us to rigorously prove the end-to-end robustness of \framework (Theorem~\ref{thm} in Section~\ref{sec-robustness-analysis}). We will show that \framework will always perform robust detection or issue an alert on objects certified by our provable analysis (Algorithm~\ref{alg-dpg-analysis} in Section~\ref{sec-robustness-analysis}). We note that this robustness property is agnostic to attack strategies and holds for \textit{any patch hiding attacker within our threat model, including adaptive attackers who have full access to our defense setup.} This strong theoretical guarantee addresses \textit{Challenge 1}. Next, in contrast to most security-critical systems whose robustness comes at the cost of clean performance, \framework achieves provable robustness for ``free" (at a negligible cost of clean performance). In \defense, we design error mitigation mechanisms to handle the imperfection of the adapted robust classifier. In \matcher, our explaining strategy ensures that even when our \defense fails to predict high objectness (missing objects; false negatives), \framework still performs as well as state-of-the-art \base.\footnote{\defense can also have other types of errors. However, we can optimize its configuration to ensure most errors are false-negatives, which our objectness explaining strategy can tolerate. More discussions are in Section~\ref{sec-defense}. We also note that we manage to build a system with high clean performance (i.e., \framework) despite the use of a module with poor clean performance (i.e., provably robust image classifier). We provide additional discussion on this intriguing property in Appendix~\ref{apx-why}.}  These designs solve \textit{Challenge 2}.

\textbf{Evaluating the first provable robustness against patch hiding attacks.} We extensively evaluate \framework performance on the PASCAL VOC~\cite{voc}, MS COCO~\cite{coco}, and KITTI~\cite{geiger2013vision} datasets. In our evaluation, we instantiate the \base with YOLOv4~\cite{bochkovskiy2020yolov4,wang2020scaled}, Faster R-CNN~\cite{ren2015faster}, and a hypothetical object detector that is perfect in the clean setting. We build \defense by adapting multiple variants of robust image classifiers~\cite{xiang2020patchguard,zhang2020clipped}. 
Our evaluation shows that our defense has a minimal impact (<1\%) on the clean performance and achieves the first provable robustness against patch hiding attacks. Our code is available at \url{https://github.com/inspire-group/DetectorGuard}.

Our contributions can be summarized as follows.
\begin{itemize}
    \item We solve two major challenges in adapting robust image classifiers for robust object detection via a careful design of \defense and \matcher.
    \item We formally prove the robustness guarantee of \framework on certified objects against any adaptive attacker within our threat model. 
    \item We extensively evaluate our defense on the PASCAL VOC~\cite{voc}, MS COCO~\cite{coco}, and KITTI~\cite{geiger2013vision} datasets and demonstrate \textit{the first provable robustness against patch hiding attacks} and a similar clean performance as conventional object detectors.
\end{itemize}

\section{Background and Problem Formulation}\label{sec-preliminary}

In this section, we introduce the object detection task, the patch hiding attack, the defense formulation, and the key principles for building provably robust image classifiers that we will adapt for robust object detection in \framework.  

\subsection{Object Detection}\label{sec-obj-det}
\textbf{Detection objective.} An object detector aims to predict a list of bounding boxes (and class labels) for all objects in the image $\mathbf{x}\in[0,1]^{W\times H \times C}$, where pixel values are rescaled into $[0,1]$, and $W,H,C$ is the image width, height, and channel, respectively. 
Each bounding box $\mathbf{b}$ is represented as a tuple $(x_{\min},y_{\min},x_{\max},y_{\max},l)$, where $x_{\min},y_{\min},x_{\max},y_{\max}$ together illustrate the coordinates of the bounding box, and $l \in \mathbf{L}=\{0,1,\cdots,N-1\}$ denotes the predicted object label ($N$ is the number of object classes).\footnote{Conventional object detectors usually output objectness score and prediction confidence as well---we discard them in notation for simplicity.}

\textbf{Conventional object detector.} Object detection models can be categorized into two-stage and one-stage detectors depending on their pipelines. A two-stage object detector first generates proposals for regions that might contain objects and then uses the proposed regions for object classification and bounding-box regression. Representative examples include Faster R-CNN~\cite{ren2015faster} and Mask R-CNN~\cite{he2017mask}. On the other hand, a one-stage object detector does detection directly on the input image without any explicit region proposal step. SSD~\cite{liu2016ssd}, YOLO~\cite{redmon2016you,redmon2017yolo9000,redmon2018yolov3,bochkovskiy2020yolov4,wang2020scaled}, RetinaNet~\cite{lin2017focal}, and EfficientDet~\cite{tan2020efficientdet} are representative one-stage detectors.

Conventionally, a detection is considered correct when 1) the predicted label matches the ground truth and 2) the overlap 
between the predicted bounding box and the ground-truth box, measured by Intersection over Union (IoU), exceeds a certain threshold $\tau$. We term a correct detection a \textit{true positive (TP)}. On the other hand, any predicted bounding box that fails to satisfy both two TP criteria is considered as a \textit{false positive (FP)}. Finally, if a ground-truth object is not detected by any TP bounding box, it is a \textit{false negative (FN)}.

\subsection{Attack Formulation}
\textbf{Attack objective.} The hiding attack~\cite{liu2019dpatch,zhao2019seeing,thys2019fooling,xu2020adversarial,wu2019making}, also referred to as the false-negative (FN) attack, aims to make object detectors miss the detection of certain objects (which increases FN) at the test time. The hiding attack can cause serious consequences in scenarios like an autonomous vehicle missing a pedestrian. Therefore, defending against patch hiding attacks is of great importance.

\textbf{Attacker capability.} We allow the localized adversary to arbitrarily manipulate pixels within one restricted region.\footnote{Provably robust defenses against \textit{one} single patch are currently an open/unsolved problem, and hence the focus of this paper. In Appendix~\ref{apx-multiple-patch}, we will justify our one-patch threat model and quantitatively discuss the implication of multiple patches.} Formally, we can use a binary \textit{pixel mask} $\mathbf{pm} \in \{0,1\}^{W\times H}$ to represent this restricted region, where the pixels within the region are set to $1$. The adversarial image then can be represented as $\mathbf{x}^\prime = (\mathbf{1}-\mathbf{pm})\odot \mathbf{x} + \mathbf{pm} \odot \mathbf{x}^{\prime\prime}$ where $\odot$ denotes the element-wise product operator, and $\mathbf{x}^{\prime\prime}\in [0,1]^{W\times H \times C}$ is the content of the adversarial patch, which the adversary can arbitrarily modify. $\mathbf{pm}$ is a function of patch size and patch location. The patch size should be limited such that the object is recognizable by a human. For patch locations, we consider three different threat models: \textit{over-patch, close-patch, far-patch}, where the patch is over, close to, or far away from the victim object, respectively. The adversary can pick any valid location within the threat model for an optimal attack.

Previous works~\cite{liu2019dpatch,saha2020role,lee2019physical} have shown that attacks against object detectors can succeed even when the patch is far away from the victim object. Therefore, defending against all three threat models is of interest.

\subsection{Defense Formulation}\label{sec-defense-formulation}

\textbf{Defense objective.} We focus on defending against patch hiding attacks. We consider our defense to be robust on an object if we can 1) detect the object on the clean image is correct and 2) detect part of the object or send out an attack alert on the adversarial image.\footnote{We note that in the adversarial setting, we only require the predicted bounding box to cover part of the object because it is likely that only a small part of the object is recognizable due to the adversarial patch (e.g., the left dog in the right part of Figure~\ref{fig-overview}). We provide additional justification for our defense objective in Appendix~\ref{apx-defense-obj}.}

Crucially, we design our defense to be \textit{provably robust}: for an object certified by our provable analysis, our defense can either detect the certified object or issue an alert regardless of what the adversary does (including any adaptive attack at any patch location within the threat model).
This robustness property is agnostic to the attack algorithm and holds against an adversary that has full access to our defense setup.

\textbf{Remark: primary focus on hiding attacks.} In this paper, we focus on the hiding attack because it is the most fundamental and notorious attack against object detectors. We can visualize dividing the object detection task into two steps: 1) detecting the object bounding box and then 2) classifying the detected object. If the first step is compromised by the hiding attack, there is no hope for robust object detection. On the other hand, \textit{securing the first step against the patch hiding attack lays a foundation for the robust object detection}; we can design effective remediation for the second step if needed (Section~\ref{sec-discussion}).

Take the application domain of autonomous vehicles (AV) as an example: an AV missing the detection of an upcoming car could lead to a serious car accident. However, if the AV detects the upcoming object but predicts an incorrect class label (e.g., mistaking a car for a pedestrian), it can still make the correct decision of stopping and avoiding the collision. Moreover, in challenging application domains where the predicted class label is of great importance (e.g., traffic sign recognition), we can feed the detected bound box to an auxiliary image classifier to re-determine the class label. The defense problem is then reduced to the robust image classification and has been studied by several previous works~\cite{xiang2020patchguard,zhang2020clipped,levine2020randomized,metzen2021efficient}. Therefore, we make the hiding attack the primary focus of this paper and will also discuss the extension of \framework against other attacks in Section~\ref{sec-discussion}. 

\subsection{Provably Robust Image Classification}\label{sec-prelim-classifier}
In this subsection, we introduce two key principles that are widely adopted in recent research on provably robust image classification against adversarial patches~\cite{zhang2020clipped,levine2020randomized,xiang2020patchguard,metzen2021efficient}. In Section~\ref{sec-defense-classifier}, we will discuss how to adapt these two principles to build a robust image classifier, which will be alter used in \defense (Section~\ref{sec-obj-predictor}).

\textbf{Feature extractor -- use small receptive fields.} The receptive field of a Deep Neural Network (DNN) is the input pixel region where each extracted feature is looking at. If the receptive field of a DNN is too large, then a small adversarial patch can corrupt most extracted features and easily manipulate the model behavior~\cite{xiang2020patchguard,liu2019dpatch,saha2020role}. On the other hand, the small receptive field bounds the number of corrupted features by $\lceil(\texttt{p}+\texttt{r}-1)/\texttt{s}\rceil$, where \texttt{p} is the patch size, \texttt{r} is the receptive field size, and \texttt{s} is the stride of receptive field (the pixel distance between two adjacent receptive field centers)~\cite{xiang2020patchguard}, and makes robust classification possible~\cite{levine2020randomized,zhang2020clipped,xiang2020patchguard,metzen2021efficient}. Popular design choices including the BagNet architecture~\cite{brendel2019approximating,zhang2020clipped,xiang2020patchguard,metzen2021efficient} and an ensemble architecture using small pixel patches as inputs~\cite{levine2020randomized,xiang2020patchguard}.

\textbf{Classification head -- do secure feature aggregation.} Given a feature map, DNN uses a \textit{classification head}, which consists of a feature aggregation layer and a fully-connected (classification) layer, to make final predictions. Since the small receptive field bounds the number of corrupted features, we can use secure aggregation techniques to build a robust classification head; design choices include clipping~\cite{zhang2020clipped,xiang2020patchguard}, masking~\cite{xiang2020patchguard}, and majority voting~\cite{levine2020randomized,metzen2021efficient}.
\section{\framework}\label{sec-defense}

In this section, we first introduce the key insight and overview of \framework, and then detail the design of our defense modules (\defense and \matcher).

\subsection{Defense Overview}\label{sec-defense-overview}

\noindent \textbf{Bridging robust image classification and robust object detection.} There has been a significant advancement in (provably) robust image classification research~\cite{chiang2020certified,zhang2020clipped,levine2020randomized,xiang2020patchguard,metzen2021efficient} while object detectors remain vulnerable. This sharp contrast motivates us to ask: \textit{can we adapt robust image classifiers for robust object detection?} 
Unfortunately, there is a huge gap between these two tasks: an image classifier only robustly predicts one single label for each image while an object detector has to robustly output a list of class labels and object bounding boxes. This gap leads to two major challenges.
\begin{itemize}
    \item \textit{Challenge 1: Lack of End-to-end Provable Robustness.} A robust image classifier only provides robustness for single-label predictions while a robust object detector requires robustness for multiple labels and bounding boxes. Therefore, an object detector adapted from a robust image classifier can still be vulnerable without any security guarantee, and we aim to carefully design our defense pipeline to enable the proof of end-to-end robustness for object detection. 
    \item \textit{Challenge 2: Amplified Cost of Clean Performance.} All existing provably robust image classifiers~\cite{chiang2020certified,zhang2020clipped,levine2020randomized,xiang2020patchguard,metzen2021efficient} attain robustness at a non-negligible cost of clean performance (e.g., >20\% clean accuracy drop on ImageNet~\cite{deng2009imagenet}), and this cost can be severely amplified when adapting towards the more demanding object detection task. An object detector with poor clean performance (even in the absence of an adversary) prohibits its real-world deployment; therefore, we aim to minimize the clean performance cost in our defense.
\end{itemize}

\textbf{\framework: an objectness explaining strategy.} In \framework, we propose an objectness explaining strategy to addresses the above two challenges. Recall that Figure~\ref{fig-overview} provides an overview of \framework, which will either output a list bounding box predictions (left figure; clean setting) or an attack alert (right figure; adversarial setting). There are three major modules in \framework: \base, \defense, and \matcher. \base is responsible for making accurate detections in the clean setting and can be any popular high-performance object detector such as YOLOv4~\cite{bochkovskiy2020yolov4,wang2020scaled} and Faster R-CNN~\cite{ren2015faster}. \defense is adapted from the core principles for building robust image classifiers as introduced in Section~\ref{sec-prelim-classifier} and aims to output a robust objectness map in the adversarial environment. We also carefully design \defense to mitigate the errors made by the robust image classifier in the clean setting. Finally, \matcher leverages predicted bounding boxes from \base to explain/match the objectness predicted by \defense and aims to catch a malicious attack. When no attack is detected, \framework will output the detection results of \base (i.e., a conventional object detector), so that our clean performance is close to state-of-the-art object detectors. When a patch hiding attack occurs, \base can miss the object while \defense can robustly predict high objectness. 
This mismatch will lead to unexplained objectness and trigger an attack alert. Notably, our objectness explaining strategy can achieve end-to-end provable robustness for ``free" (at a negligible cost of clean performance) and solve two major challenges. We will introduce the module details and theoretically analyze the free provable robustness property.

 \textbf{Algorithm Pseudocode.}
We provide the pseudocode of \framework in Algorithm~\ref{alg-dpg} and a summary of important notation in Table~\ref{tab-notation}. The main procedure $\textsc{DG}(\cdot)$ has three sub-procedures: $\textsc{BaseDetector}(\cdot)$, $\textsc{ObjPredictor}(\cdot)$, and $\textsc{DetMatcher}(\cdot)$. The sub-procedure  $\textsc{BaseDetector}(\cdot)$ can be any off-the-shelf object detector as discussed in Section~\ref{sec-obj-det}. All tensors/arrays are represented with bold symbols and scalars are in italic.
All tensor/array indices start from zeros; the tensor/array slicing is in Python style (e.g., $[i:j]$ means all indices $k$ satisfying $i\leq k<j$). We assume that the ``background" class corresponds to the largest class index.

In the remainder of this section, we first introduce how we instantiate robust image classifiers and then discuss the design of \defense and \matcher.

\begin{table}[t]
    \centering
    \caption{Summary of important notation}
    \vspace{-1em}
 \resizebox{\linewidth}{!}
  { \begin{tabular}{l|l|l|l}
    \toprule
    \textbf{Notation} & \textbf{Description} & \textbf{Notation} & \textbf{Description} 
    \\
    \midrule
 $\mathbf{x}$ & Input image & $\mathbf{b}$ &  bounding box \\
  $\mathbf{fm}$ &Feature map & $\mathbf{om}$ &Objectness map \\
   $\mathbf{v}$ &classification logits & $N$ &number of object classes\\
  $(w_x,w_y)$ & window size & $(p_x,p_y)$ & patch size \\
   $T$ &binarizing threshold & {$\mathcal{D}$ }& detection results \\
 
$\mathbf{u},\mathbf{l}$ &  \multicolumn{3}{l}{upper/lower bound of classification logits values of each class}\\
      \bottomrule
    \end{tabular}}
    \label{tab-notation}
\end{table}

\subsection{Instantiating Robust Image Classifiers}\label{sec-defense-classifier}
To start with, we discuss how we build robust image classifiers that will be used in \defense. 

As discussed in Section~\ref{sec-prelim-classifier}, we can build a robust image classifier $\textsc{RC}(\cdot)$ using a \textit{feature extractor} $\textsc{FE}(\cdot)$ with small receptive fields, and a \textit{robust classification head} $\textsc{RCH}(\cdot)$ with secure feature aggregation. In our design, we choose BagNet~\cite{brendel2019approximating} backbone as the feature extractor $\textsc{FE}(\cdot)$, and we clip elements of local logits vectors\footnote{The local logits vector~\cite{xiang2020patchguard} is the classification logits based on each local feature that is extracted from a particular region (i.e., the receptive field) of the input image.} into $[0,\infty]$ for secure aggregation in $\textsc{RCH}(\cdot)$. This implementation is similar to the robust image classifier Clipped BagNet (CBN)~\cite{zhang2020clipped}, but we note that we use a different clipping function that is tailored to our more challenging task of \textit{object detection}. In Appendix~\ref{apx-robust-classifier}, we provide additional details of $\textsc{FE}(\cdot)$ and $\textsc{RCH}(\cdot)$ and also discuss alternative design choices of robust image classifier $\textsc{RC}(\cdot)$ (e.g., robust masking from PatchGuard~\cite{xiang2020patchguard}).

\textbf{Remark: Limitations of robust classifiers.} We note that the adapted robust image classifier $\textsc{RC}(\cdot)$ achieves robustness at the cost of a non-negligible clean performance drop~\cite{xiang2020patchguard,levine2020randomized,metzen2021efficient}. Therefore, in the clean setting, classification at different image locations can be imprecise with three typical errors that lead to \textit{Challenge 2}:
\begin{itemize}
    \item \textit{Clean Error 1:} Confusion between two object class labels
    \item  \textit{Clean Error 2:} Predicts background pixels as objects 
     \item \textit{Clean Error 3:} Predicts objects as ``background"
\end{itemize}
In the next two subsections, we will discuss how \framework design can eliminate/mitigate these three clean errors.

\subsection{\defense}\label{sec-obj-predictor}
\textbf{Overview.} \defense aims to output a robust objectness map that indicates the probability of objects being present at different locations. Its high-level idea is to perform robust \textit{image classification} on different regions to predict an object class label or ``background". We provide a simplified visual overview in Figure~\ref{fig-obj}. \defense involves two major operations: robust feature-space window classification and objectness map generation. The first step aims to perform efficient robust classification at different image regions and the second step aims to filter out clean errors made by the robust classifier and generate the final objectness map.

\begin{figure}[t]
    \centering
    \includegraphics[width=\linewidth]{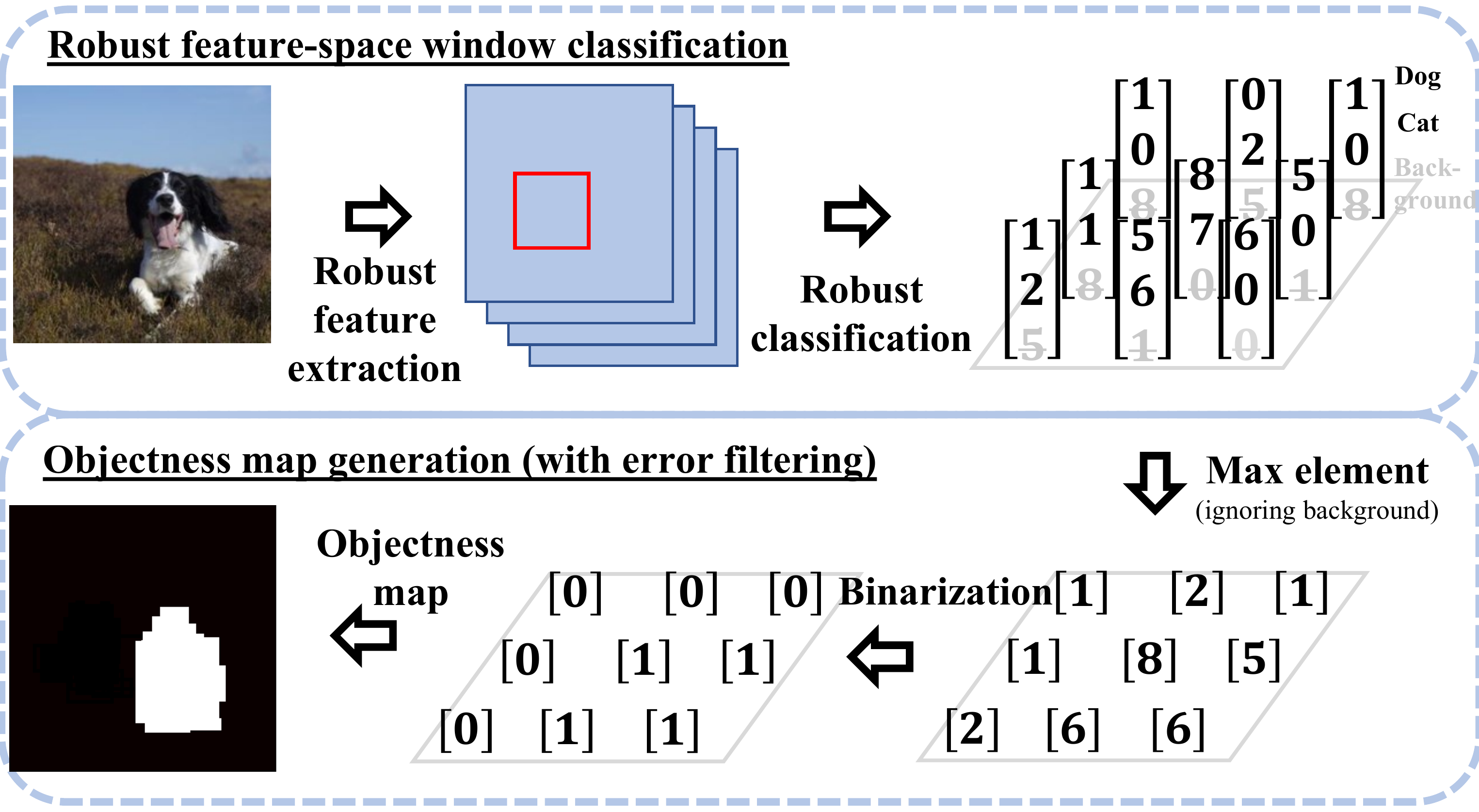}
    \caption{Visualization of Objectness Predictor operations}
    \label{fig-obj}
\end{figure}

\textbf{Robust feature-space window classification.} To perform robust image classification at different image locations, we first extract a feature map via $\textsc{FE}(\cdot)$, and then apply the robust image classification head $\textsc{RCH}(\cdot)$ to a sliding window over the feature map.

\begin{algorithm}[!t]
\caption{\framework}\label{alg-dpg}
\begin{algorithmic}[1]
\renewcommand{\algorithmicrequire}{\textbf{Input:}}
\renewcommand{\algorithmicensure}{\textbf{Output:}}
\Require input image $\mathbf{x}$, window size $(w_x,w_y)$, binarizing threshold $T$, \base $\textsc{BaseDetector}(\cdot)$, robust classifier $\textsc{RC}(\cdot)$ (consisting of a feature extractor $\textsc{FE}(\cdot)$ and a robust classification head $\textsc{RCH}(\cdot)$), cluster detection procedure $\textsc{DetCluster}(\cdot)$
\Ensure  robust detection $\mathcal{D}^*$ or $\texttt{ALERT}$
\Procedure{DG}{$\mathbf{x},w_x,w_y,T$}
\State ${\mathcal{D}}\gets \textsc{BaseDetector}(\mathbf{x})$\Comment{Conventional detection}
\State $\mathbf{om} \gets \textsc{ObjPredictor}(\mathbf{x},w_x,w_y,T)$ \Comment{Objectness}
\State $a\gets\textsc{ObjExplainer}({\mathcal{D}},\mathbf{om})$ \Comment{Detect hiding attacks}
\If{${a}== \texttt{True}$}\label{ln-output-s}\Comment{Malicious mismatch}
\State$\mathcal{D}^*\gets\texttt{ALERT}$\Comment{Trigger an alert}
\Else
\State $\mathcal{D}^*\gets{\mathcal{D}}$\label{ln-return}\Comment{Return \base's predictions}
\EndIf
\State\Return $\mathcal{D}^*$\label{ln-output-e}
\EndProcedure

\item[]
\Procedure{ObjPredictor}{$\mathbf{x},w_x,w_y,T$}\label{ln-obj-pred-s}
\State $\mathbf{fm} \gets \textsc{FE}(\mathbf{x})$\label{ln-fe} \Comment{Extract feature map}
\State $X,Y,\_ \gets \textsc{Shape}(\mathbf{fm})$ \Comment{Get the shape of $\mathbf{fm}$}
\State $\bar{\mathbf{om}} \gets \textsc{ZeroArray}[X,Y,N+1]$\label{ln-op-empty}\Comment{Initialization}
\For{each valid $(i,j)$}\label{ln-rc-s}\Comment{Every window location}
\State $l,\mathbf{v} \gets\textsc{RCH}(\mathbf{fm}[i:i+w_x,j:j+w_y])$\label{ln-rc-e} \Comment{Classify}

\LeftCommentb{Add classification logits}
\State $\bar{\mathbf{om}}[i\text{ : }i+w_x,j\text{ : }j+w_y]\gets\bar{\mathbf{om}}[i\text{ : }i+w_x,j\text{ : }j+w_y]+\mathbf{v}$\label{ln-op-om}
\EndFor
\State $\bar{\mathbf{om}}\gets \textsc{MaxObj}(\bar{\mathbf{om}},\texttt{axis}=-1)$\label{ln-max}\Comment{Max objectness score}
\State ${\mathbf{om}}\gets \textsc{Binarize}(\bar{\mathbf{om}},T\cdot w_x\cdot w_y)$\label{ln-bin}\Comment{Binarization}
\State\Return $\mathbf{om}$
\EndProcedure\label{ln-obj-pred-e}
\item[]

\Procedure{ObjExplainer}{$\mathcal{D},\mathbf{om}$}
\State $\mathbf{\hat{om}}\gets \textsc{Copy}(\mathbf{om})$ \Comment{A copy of $\mathbf{om}$}\label{ln-match-s}
\LeftCommenta{Match each detected box to objectness map}
\For {$i\in\{0,1,\cdots,|\mathcal{D}|-1\}$}
\State $x_{\min},y_{\min},x_{\max},y_{\max},l \gets \mathbf{b} \gets {\mathcal{D}}[i]$
\If {$\textsc{Sum}({\mathbf{om}}[x_{\min}:x_{\max},y_{\min}:y_{\max}])>0$}
\State ${\mathbf{\hat{om}}}[x_{\min}:x_{\max},y_{\min}:y_{\max}])\gets\mathbf{0}$
\EndIf
\EndFor\label{ln-match-e}
\If{$\textsc{DetCluster}({\mathbf{\hat{om}}})$ is \texttt{None}}
\State \Return\texttt{False}\Comment{All objectness explained}
\Else
\State \Return\texttt{True}\Comment{Unexplained objectness}
\EndIf
\EndProcedure

\end{algorithmic} 
\end{algorithm}

\textit{Pseudocode}. The pseudocode of \defense in Line~\ref{ln-obj-pred-s}-\ref{ln-obj-pred-e} of Algorithm~\ref{alg-dpg}. We first extract the feature map $\mathbf{fm}$ with $\textsc{FE}(\cdot)$ (Line~\ref{ln-fe}). Next, for every valid window location, represented as $(i,j)$,\footnote{We will use feature-space coordinates for the remainder of the paper. The mapping between pixel-space and feature-space coordinates is discussed in Appendix~\ref{apx-mapping}.} we feed the feature window $\mathbf{fm}[i:i+w_x,j:j+w_y]$ to a robust classification head $\textsc{RCH}(\cdot)$ to get the classification label $l$ and the classification logits $\mathbf{v}\in\mathbb{R}^{N+1}$ for $N$ object classes and the ``background" class (Line~\ref{ln-rc-e}). The use of $\textsc{RCH}(\cdot)$ ensures that the classification is robust when window features are corrupted.

\textit{Remark: defense efficiency.} We note that our window classification operates in the feature space, and this allows us to reuse the expensive feature map generation (i.e., $\textsc{FE}(\cdot)$); each classification only needs a cheap computation of the classification head (i.e., $\textsc{RCH}(\cdot)$). Therefore, our defense only incurs a small overhead (will be evaluated in Section~\ref{sec-eval-hyp}).

\textbf{Objectness map generation: handling clean errors of robust classifiers.} Next, we aim to filter out incorrect window classifications and generate the final objectness map. We reduce each prediction vector at each location to its maximum non-background element to discard label information (eliminating \textit{Clean Error 1}) and perform binarization (with a threshold) to remove objectness predicted with low confidence (mitigating \textit{Clean Error 2}).

\textit{Pseudocode.} First, we initialize an all-zero $\mathbb{R}^{N+1}$ vector ($N$ object classes plus the ``background" class) at every feature location and use one tensor $\bar{\mathbf{om}}$ to represent all vectors (Line~\ref{ln-op-empty}). Second, we aim to gather all window classification results: for each window, we add the classification logits $\mathbf{v}$ to every vector located within the window (Line~\ref{ln-op-om}). Third, we take the maximum non-background element in each vector as the objectness score at the corresponding location (Line~\ref{ln-max}). This operation discards label information and fully eliminates \textit{Clean Error 1}, e.g., confusion between bicycles and motorbikes. At last, we binarize every objectness score (Line~\ref{ln-bin}): if the score is larger than $T\cdot w_x \cdot w_y$, we set it to one; otherwise, set it to zero. This binarization mitigates \textit{Clean Error 2}, when the classifier incorrectly predicts background as objects but with low classification confidence. We discuss the strategy for \textit{Clean Error 3} in the next subsection.

\begin{figure}[t]
    \centering
    \includegraphics[width=\linewidth]{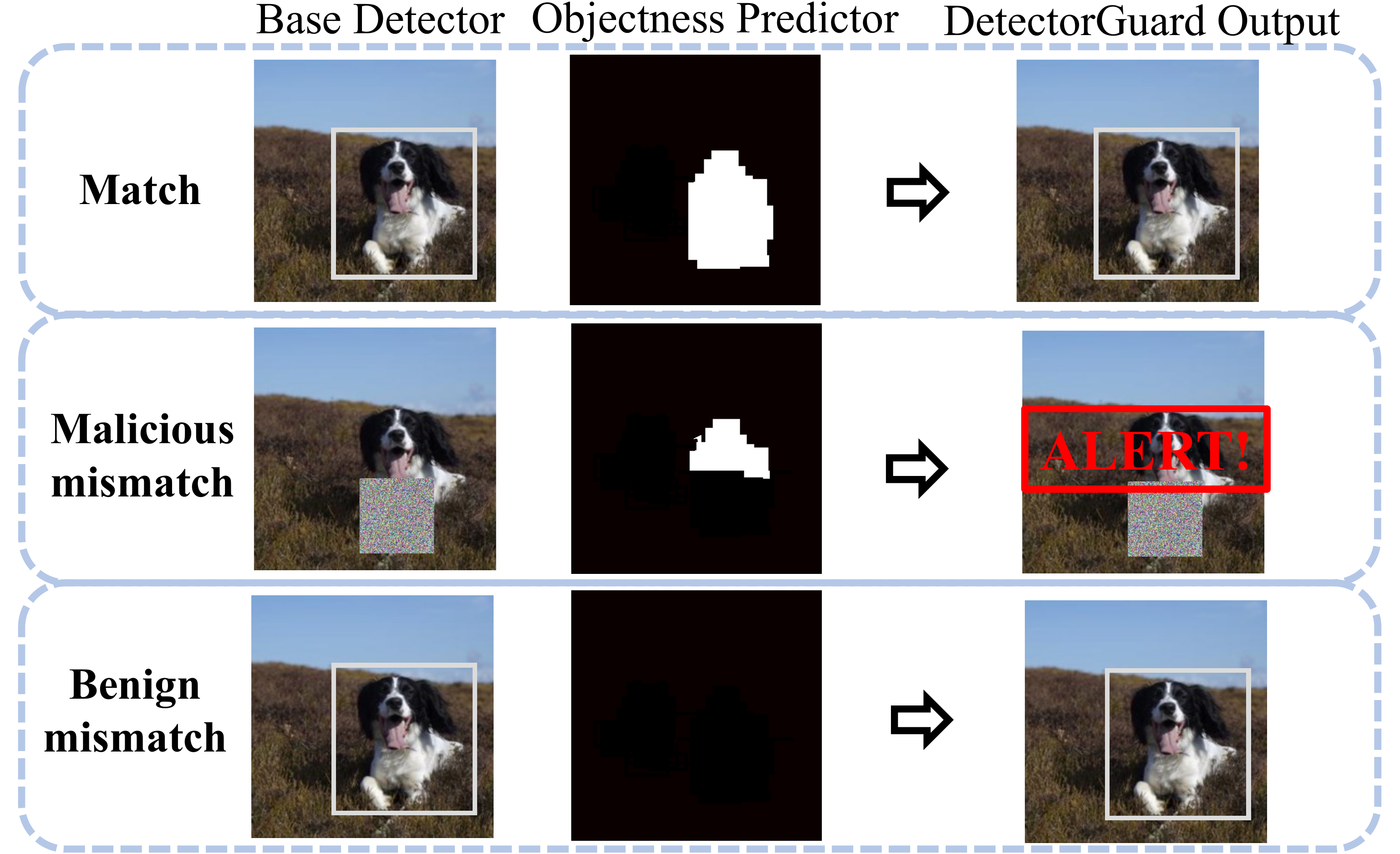}
    \caption{Visualization of explaining/matching rules}
   \label{fig-match}
\end{figure}
\subsection{\matcher}\label{sec-matcher}
\matcher takes as inputs the predicted bounding boxes of \base and the generated objectness map of \defense, and tries to use each predicted bounding box to explain/match the high activation in the objectness map. Its outcome determines the final prediction of \framework. We will first introduce the high-level explaining/matching rules and then elaborate on the algorithm.

 \textbf{Explaining/matching rules.} 
There are three possible explaining/matching outcomes, each of them leading to a different prediction strategy (a visual example is in Figure~\ref{fig-match}):
\begin{itemize}
    \item A \textbf{match} happens when \base and \defense both predict a bounding box or high objectness at a specific location. In this simplest case, \textit{the objectness is well explained by the bounding box}; our defense will consider the detection as correct and output the accurate bounding box and the class label predicted by \base.
    \item A \textbf{malicious mismatch} will be flagged when only \defense outputs high objectness. This is most likely to happen when a hiding attack fools the conventional object detector to miss the object while our \defense still makes robust predictions. In this case, our defense will find \textit{unexplained objectness} and send out an attack alert.
    \item A \textbf{benign mismatch} occurs when only \base detects the object and there is \textit{no objectness to be explained}. This can happen when \defense incorrectly misses the object due to its limitations (recall our remark in Section~\ref{sec-obj-predictor}). In this case, we trust \base and output its predicted bounding box. Notably, this strategy can fully eliminate \textit{Clean Error 3}, i.e., predicting objects as background.\footnote{We note that this miss can also be caused by other attacks that are orthogonal to the focus of this paper, e.g., FP attacks that aim to introduce incorrect bounding box predictions. We will discuss such attacks and our defense strategies in Section~\ref{sec-discussion}.}
\end{itemize}  
Next, we discuss the concrete procedure for explaining objectness.

\textbf{Matching and explaining objectness.} In Line \ref{ln-match-s}-\ref{ln-match-e} of Algorithm~\ref{alg-dpg}, we use each predicted bounding box to match/explain the objectness predicted at the same location. First, we create a copy of $\mathbf{om}$ as $\mathbf{\hat{om}}$ to hold the explaining results. Next, for each bounding box $\mathbf{b}$, we get its coordinates $x_{\min},y_{\min},x_{\max},y_{\max}$, and calculate the sum of objectness scores within the same box on the objectness map $\mathbf{{om}}$. If the objectness sum is larger than zero, we assume that the bounding box $\mathbf{b}$ agrees with $\mathbf{{om}}$, and we zero out the corresponding region in ${\mathbf{\hat{om}}}$, to indicate that this region of objectness has been \textit{explained} by the detected bounding box. On the other hand, if all objectness scores are zeros, we assume it is a benign mismatch; the algorithm proceeds without alert.

\textbf{Detecting clusters of unexplained objectness.} The final step is to detect unexplained objectness in $\mathbf{\hat{om}}$. We use the sub-procedure $\textsc{DetCluster}(\cdot)$ to determine if any non-zero points in ${\mathbf{\hat{om}}}$ form a large cluster. Specifically, we choose DBSCAN~\cite{ester1996density} as the cluster detection algorithm, which will assign each point to a certain cluster or label it as an outlier based on the point density in its neighborhood. If $\textsc{DetCluster}({\mathbf{\hat{om}}})$ returns \texttt{None}, it means that no large cluster is found, or all objectness predicted by \defense is explained by the bounding boxes predicted by \base; $\textsc{ObjExplainer}(\cdot)$ then returns \texttt{False} (i.e., no attack detected). We note that this clustering operation further mitigates \textit{Clean Error 2} when the robust classifier predicts background as objects at only a few scattered locations. On the other hand, receiving a non-empty cluster set indicates that there are clusters of unexplained objectness activations in ${\mathbf{\hat{om}}}$ (i.e, \base misses an object but \defense predicts high objectness). \matcher regards this as a sign of patch hiding attacks and returns \texttt{True}.

\textbf{Final output.} Line \ref{ln-output-s}-\ref{ln-output-e} of Algorithm~\ref{alg-dpg} demonstrates the strategy for the final prediction. If the alert flag $a$ is \texttt{True} (i.e., a malicious mismatch is detected), \framework returns $\mathcal{D}^* =\texttt{ALERT}$. In other cases, \framework returns the detection $\mathcal{D}^* = \mathcal{D}$.

\textbf{Remark: Clean performance of \framework.} Recall that \textit{Clean Error 1} of the robust classifier is \textit{fully eliminated} in our objectness map generation via discarding label information; \textit{Clean Error 2} is \textit{mitigated} via binarizing (in \defense) and clustering (in \matcher) operations; \textit{Clean Error 3} is \textit{fully tolerated} via our prediction matching strategy (the benign mismatch case). Therefore, we can safely optimize the setting of \framework to mitigate most of \textit{Clean Error 2} (which can lead to unexplained objectness in the clean setting and trigger a false alert) so that we can achieve a clean performance that is comparable to state-of-the-art object detectors (performance difference smaller than 1\%; more details are in Section~\ref{sec-evaluation}). This helps us solve \textit{Challenge 2: Amplified Cost of Clean Performance}. In the next section, we will demonstrate that our efforts in mitigating the imperfection of robust classifiers are worthwhile by showing how \framework addresses \textit{Challenge 1: Lack of End-to-end Provable Robustness}.

\section{End-to-end Provable Robustness}\label{sec-robustness-analysis}
\begin{algorithm}[!t]
\caption{Provable Analysis of \framework}\label{alg-dpg-analysis}
\begin{algorithmic}[1]
\renewcommand{\algorithmicrequire}{\textbf{Input:}}
\renewcommand{\algorithmicensure}{\textbf{Output:}}
\Require input image $\mathbf{x}$, window size $(w_x,w_y)$, matching threshold $T$, the set of patch locations $\mathcal{P}$, the object bounding box $\mathbf{b}$, feature extractor $\textsc{FE}(\cdot)$, provable analysis of the robust classification head $\textsc{RCH-PA}(\cdot)$, cluster detection procedure $\textsc{DetCluster}(\cdot)$
\Ensure  whether the object $\mathbf{b}$ in $\mathbf{x}$ has provable robustness

\Procedure{DG-PA}{$\mathbf{x},w_x,w_y,T,\mathcal{P},\mathbf{b}$}
\If{$\mathbf{b}\not\in\textsc{DG}(\mathbf{x},w_x,w_y,T)$}\label{ln-clean-s} 
\State\Return \texttt{False}\Comment{Clean detection is incorrect}
\EndIf\label{ln-clean-e}
\State $\mathbf{fm} \gets \textsc{FE}(\mathbf{x})$ \Comment{Extract feature map}
\For{each $\mathbf{p}\in\mathcal{P}$}\Comment{Check every patch location}
\State $x,y,p_x,p_y\gets\mathbf{p}$
\State $r \gets \textsc{DG-PA-One}(\mathbf{fm},x,y,w_x,w_y,p_x,p_y,\mathbf{b},T)$
\If{$r==\texttt{False}$} 
\State \Return \texttt{False}\Comment{Possibly vulnerable}
\EndIf
\EndFor
\State \Return \texttt{True}\Comment{Provably robust} \label{ln-robust}
\EndProcedure

\item[]

\Procedure{DG-PA-One}{$\mathbf{fm},x,y,w_x,w_y,p_x,p_y,\mathbf{b},T$}

\State $X,Y,\_ \gets \textsc{Shape}(\mathbf{fm})$ \Comment{Get the shape of $\mathbf{fm}$}
\State $\bar{\mathbf{om}}_*\gets \textsc{ZeroArray}[X,Y,N+1]$ \Comment{Initialization}
\LeftCommenta{Generates worst-case objectness map for analysis}
\For{each valid $(i,j)$}\label{ln-analysis-om-s}\Comment{Every window location}
\State $\mathbf{u},\mathbf{l}\gets\textsc{RCH-PA}({\mathbf{fm}[i:i+w_x,j:j+w_y]},x-i,y-j,p_x,p_y)$
\LeftCommentb{Add worst-case (lower-bound) logits}
\State $\bar{\mathbf{om}}_*[i:i+w_x,j:j+w_y]\gets\bar{\mathbf{om}}_*[i:i+w_x,j:j+w_y]+\mathbf{l}$
\EndFor\label{ln-analysis-om-e}
\State $\bar{\mathbf{om}}_*\gets \textsc{MaxObj}(\bar{\mathbf{om}}_*,\texttt{axis}=-1)$\Comment{Max objectness score}
\State ${\mathbf{om}}_*\gets\textsc{Binarize}(\bar{\mathbf{om}}_*,T\cdot w_x \cdot w_y)$\Comment{Binarization}
\State $x_{\min},y_{\min},x_{\max},y_{\max},l \gets \mathbf{b} $
\If{$\textsc{DetCluster}(\mathbf{om}_*[x_{\min}:x_{\max},y_{\min}:y_{\max}])$ is \texttt{None}}
\State\Return\texttt{False}\Comment{No high objectness left}
\Else
\State \Return \texttt{True}\Comment{High worst-case objectness}
\EndIf
\EndProcedure
\end{algorithmic}
\end{algorithm}

Recall that we consider \framework to be provably robust for a given object (in a given image) when it can make correct detection on the clean image and will either detect part of the object or issue an alert on the adversarial image. The robustness property holds for \emph{any adaptive patch hiding attacker at any location} within our threat model, including ones who have full access to our models and defense setup. In this section, we will prove the end-to-end robustness of \framework, solving \textit{Challenge 1}.

\textbf{Provable analysis of \framework.} Thanks to our objectness explaining strategy, a patch hiding attacker has to make \textit{both} \base and \defense fail to predict a bounding box, or high objectness, for a successful attack. Therefore, if we can prove that \defense can output high objectness for an object in the worst case, we can certify its provable robustness. We present the provable analysis of \framework in Algorithm~\ref{alg-dpg-analysis}. The algorithm takes a clean image $\mathbf{x}$, a ground-truth object bounding box $\mathbf{b}$,\footnote{We note that the ground-truth information is essential in our provable analysis (Algorithm~\ref{alg-dpg-analysis}) but is not used in our actual defense (Algorithm~\ref{alg-dpg}).} and a set of valid \textit{patch locations} $\mathcal{P}$ as inputs, and will determine whether the object in bounding box $\mathbf{b}$ in the image $\mathbf{x}$ has provable robustness against any patch at any location in $\mathcal{P}$. We state the correctness of Algorithm~\ref{alg-dpg-analysis} in Theorem~\ref{thm}, and will explain the algorithm details by proving the theorem.

\begin{ttheorem}\label{thm}
Given an object bounding box $\mathbf{b}$ in a clean image $\mathbf{x}$, a set of patch locations $\mathcal{P}$, window size $(w_x,w_y)$, and binarizing threshold $T$ (used in $\textsc{DG}(\cdot)$ of Algorithm~\ref{alg-dpg}), if Algorithm~\ref{alg-dpg-analysis} returns \texttt{True}, i.e., $\textsc{DG-PA}(\mathbf{x},w_x,w_y,T,\mathbf{b},\mathcal{P})=\texttt{True}$, \framework has provable robustness for the object $\mathbf{b}$ against any patch hiding attack using any patch location in $\mathcal{P}$.
\end{ttheorem}

\begin{proof}
$\textsc{DG-PA}(\cdot)$ first calls $\textsc{DG}(\cdot)$ of Algorithm~\ref{alg-dpg} to determine if \framework can detect the object bounding box $\mathbf{b}$ on the clean image $\mathbf{x}$. The algorithm will proceed only when the clean detection is correct (Line~\ref{ln-clean-s}-\ref{ln-clean-e}).

Next, we extract feature map $\mathbf{fm}$, iterate over each patch location in $\mathcal{P}$, and call the sub-procedure $\textsc{DG-PA-One}(\cdot)$, which analyzes worst-case behavior over all possible adversarial strategies for each patch location, to determine the model robustness. If any call of $\textsc{DG-PA-One}(\cdot)$ returns \texttt{False}, the algorithm returns \texttt{False}, indicating that at least one patch location can bypass our defense. On the other hand, if the algorithm tries all valid patch locations and does not return \texttt{False}, this means that \framework is provably robust to all patch locations in $\mathcal{P}$ and the algorithm returns \texttt{True}.

In the sub-procedure $\textsc{DG-PA-One}(\cdot)$, we analyze the worst-case output of \defense against the given patch location. We perform the provable analysis of the robust image classification head (using $\textsc{RCH-PA}(\cdot)$) to determine the lower/upper bounds of classification logits for each window. If the aggregated worst-case (i.e., lower bound) objectness map still has high activation for the object of interest, we can certify the robustness of \framework. 

As shown in the $\textsc{DG-PA-One}(\cdot)$ pseudocode, we first initialize a zero array $\bar{\mathbf{om}}_*$ to hold the worst-case objectness scores. We then iterate over each sliding window and call $\textsc{RCH-PA}(\cdot)$, which takes the feature map window $\mathbf{fm}[i:i+w_x,j:j+w_y]$, relative patch coordinates $(x-i,y-j)$, patch size $(p_x,p_y)$ as inputs and outputs the upper bound $\mathbf{u}$ and lower bound $\mathbf{l}$ of the classification logits. Since the goal of the hiding attack is to minimize the objectness scores, we only need to reason about the lower bound of classification logits. Recall that in $\textsc{RCH}(\cdot)$, we clip all local logits values into $[0,\infty]$; therefore, the best an adversary can do is to push all corrupted logits values down to zeros. We then compute the lower bound $\mathbf{l}$ by zeroing out all corrupted logits values and aggregating the remaining ones. We note that the sub-procedure $\textsc{DG-PA-One}(\cdot)$ aims to check defense robustness for a particular patch location; therefore, the patch location and corrupted features/logits are known in this provable analysis (we discuss how to map pixel-space coordinates to feature-space coordinates in Appendix~\ref{apx-mapping}).

Given the lower bound $\mathbf{l}$ of every window classification logits, we will add it to the corresponding region of $\bar{\mathbf{om}}_*$. After we analyze all valid windows, we call $\textsc{MaxObj}(\cdot)$ and $\textsc{Binarize}(\cdot)$ for the worst-case objectness map ${\mathbf{om}}_*$. We then get the cropped objectness map that corresponds to the object of interest (i.e., ${\mathbf{om}_*}[x_{\min}:x_{\max},y_{\min}:y_{\max}]$) and feed it to the cluster detection algorithm $\textsc{DetClutser}(\cdot)$. If \texttt{None} is returned, a hiding attack using this patch location might succeed, and the sub-procedure returns \texttt{False}. Otherwise, \defense has a high worst-case objectness and is thus robust to any attack using this patch location. This implies the provable robustness, and the sub-procedure returns \texttt{True}. 
\end{proof}

Theorem~\ref{thm} shows that if our provable analysis (Algorithm~\ref{alg-dpg-analysis}) returns \texttt{True} for certain objects, \framework (Algorithm~\ref{alg-dpg}) will always detect the certified object or issue an attack alert. This robustness property is agnostic to attack strategies and holds for any adaptive attacker at any location within our threat model. In our evaluation (next section), we will use Algorithm~\ref{alg-dpg-analysis} and Theorem~\ref{thm} to certify the provable robustness of every object in every image and report the percentage of certified objects.

\begin{table*}[t]
    \centering
    \caption{Clean performance of \framework}
    \vspace{-1em}
    \resizebox{\linewidth}{!}
{ \small\begin{tabular}{c|c|c|c|c|c|c|c|c|c}
    \toprule
    & \multicolumn{3}{c|}{PASCAL VOC}&\multicolumn{3}{c|}{MS COCO}&\multicolumn{3}{c}{KITTI}\\
    & vanilla AP& defense AP& FAR@0.8& vanilla AP& defense AP& FAR@0.6& vanilla AP& defense AP& FAR@0.8\\
    \midrule
    Perfect clean detector&100\%&99.3\%&0.9\% &100\%&99.0\%&1.2\%&100\%&99.0\% & 1.5\%\\
    YOLOv4&92.9\%&92.4\%&4.0\%&73.6\%&73.4\%&1.6\%&93.1\%&92.4\%&1.7\%\\
    Faster R-CNN&90.0\%&89.6\%&2.9\%&66.7\%&66.5\%&0.9\%&89.9\%&89.1\%&1.4\%\\
      \bottomrule
    \end{tabular}}
    \label{tab-clean-dpg}
\end{table*}
\section{Evaluation}\label{sec-evaluation}

In this section, we provide a comprehensive evaluation of \framework on PASCAL VOC~\cite{voc}, MS COCO~\cite{coco}, and KITTI~\cite{geiger2013vision} datasets. We will first introduce the datasets and models used in our evaluation, followed by our evaluation metrics. We then report our main evaluation results on different models and datasets, and finally provide a detailed analysis of \framework performance. Our code is available at \url{https://github.com/inspire-group/DetectorGuard}.

\subsection{Dataset and Model}

\noindent\textbf{Dataset:}

\underline{PASCAL VOC}~\cite{voc}. The detection challenge of PASCAL Visual Object Classes (VOC) project is a popular object detection dataset with annotations for 20 different classes. We take \texttt{trainval2007} (5k images) and \texttt{trainval2012} (11k images) as our training set and evaluate our defense on \texttt{test2007} (5k images), which is a conventional usage of the PASCAL VOC dataset~\cite{liu2016ssd,zhang2019towards}.

\underline{MS COCO}~\cite{coco}. The Microsoft Common Objects in COntext (COCO) dataset is an extremely challenging object detection dataset with 80 annotated common object categories. We use the training and validation set of \texttt{COCO2017} for our experiments. The training set has 117k images, and the validation set has 5k images. We ignore bounding boxes with the flag \texttt{iscrowd=1} for simplicity.

\underline{KITTI}~\cite{geiger2013vision}. KITTI is an autonomous vehicle dataset that contains both 2D camera images and 3D point clouds. We take its 7481 2D images and use 80\% of randomly splited images for training and the remaining 20\% for validation. We merge all classes into three classes: \texttt{car} (all different classes of vehicles), \texttt{pedestrian}, \texttt{cyclist}.

\noindent \textbf{\base:}

\underline{YOLOv4}~\cite{bochkovskiy2020yolov4,wang2020scaled} is the state-of-the-art one-stage object detector. We choose Scaled-YOLOv4-P5~\cite{wang2020scaled} in our evaluation. For MS COCO, we use the pre-trained model. For PASCAL VOC and KITTI, we fine-tune the model previously trained on MS COCO. 

\underline{Faster R-CNN}~\cite{ren2015faster} is a representative two-stage object detector. We use ResNet101-FPN as its backbone network. Image pre-processing and model architecture follows the original paper. We use pre-trained models for MS COCO and fine-tune pre-trained models for PASCAL VOC and KITTI detectors.

\underline{Perfect Clean Detector} (PCD) is a hypothetical object detector simulated with ground-truth annotations. PCD can always make correct detection in the clean setting but is assumed vulnerable to patch hiding attacks. This hypothetical object detector ablates the errors of \base and helps us better understand the behavior of \defense and \matcher.

\noindent \textbf{\defense:}

\underline{BagNet-33}~\cite{brendel2019approximating}, which has a 33$\times$33 small receptive field, is the backbone network of \defense. For PASCAL VOC and MS COCO, we zero-pad each image to a square and resize it to 416$\times$416 before feeding it to BagNet; for KITTI, we resize each image to 224$\times$740. We fine-tune a BagNet model that is pre-trained on ImageNet~\cite{deng2009imagenet}. To train an image classifier given a list of bounding boxes in the object detection dataset, we first map pixel-space bounding boxes to the feature space (details of box mapping are in Appendix~\ref{apx-mapping}). We then teach BagNet to make correct predictions on cropped feature maps by minimizing the cross-entropy loss between aggregated feature predictions and one-hot encoded label vectors. In addition, we aggregate all features outside any feature boxes as the ``negative" feature vector for the ``background" classification.

\noindent\textbf{Default Hyper-parameter:}

We will analyze the effect of different hyper-parameters in Section~\ref{sec-eval-hyp}. In our default setting, we use a square \textit{feature-space window} of size 8 and the DBSCAN clustering~\cite{ester1996density} with $\texttt{eps}=3,\texttt{min\_points}=24$ for different datasets.\footnote{A 416$\times$416 (or 224$\times$740) pixel image results in a 48$\times$48 (or 24$\times$89) feature map using our BagNet-33 implementation.} We set the default binarizing threshold to 32 for PASCAL VOC, 36 for MS COCO, and 11 for KITTI based on different model properties with different datasets.

\subsection{Metric}\label{sec-eval-metric}

\noindent\textbf{Clean Performance Metric:}

\underline{Precision and Recall.} We calculate the precision as TP/(TP+FP) and the recall as TP/(TP+FN). For the clean images without a false alert, we follow previous works~\cite{zhang2019towards,chiang2020detection} setting the IoU threshold $\tau=0.5$ and count TPs, FPs, FNs in the conventional manner. For images that have false alerts, we set TP and FP to zeros, and FN to the number of ground-truth objects since no bounding box is predicted. We note that conventional object detectors use a confidence threshold to filter out bounding boxes with low confidence values. As a result, different confidence thresholds will give different precision and recall values; we will plot the entire precision-recall curve to analyze the model performance.

\underline{Average Precision (AP).} To remove the dependence on the confidence threshold and to have a global view of model performance, we also report Average Prevision (AP) as one of evaluation metrics. We vary the confidence threshold from 0 to 1, record the precision and recall at different thresholds, and calculate AP as the averaged precision at different recall levels. This calculated AP can be considered as an approximation of the AUC (area under the curve) for the precision-recall curve. We note that AP is one of the most widely used performance metrics in object detection benchmark competitions~\cite{voc,coco,geiger2013vision} and research papers~\cite{redmon2017yolo9000,redmon2018yolov3,bochkovskiy2020yolov4,ren2015faster,lin2017focal,tan2020efficientdet,liu2016ssd,he2017mask}.

\underline{False Alert Rate (FAR@{0.x}).} FAR is defined as the percentage of clean \textit{images} on which \framework will trigger a false alert. The false alert is mainly caused by \textit{Clean Error 2} as discussed in Section~\ref{sec-defense}. We note that FAR is also closely tied to the confidence threshold of \base: a higher confidence threshold leads to fewer predicted bounding boxes, leading to more unexplained objectness, and finally higher FAR. We will report FAR at different recall levels for a global evaluation, and use FAR@0.x to denote FAR at a \textit{clean recall} of 0.x.

\noindent\textbf{Provable Robustness Metric:}

\underline{Certified Recall (CR@0.x).} We use {certified recall} as the robustness metric against patch hiding attacks. The {certified recall} is defined as the percentage of ground-truth \textit{objects} that have provable robustness against any patch hiding attack. Recall that an object has provable robustness when Algorithm~\ref{alg-dpg-analysis} (our provable analysis) returns \texttt{True}. Note that CR is also affected by the performance of \base (e.g., confidence threshold) since its prerequisite is the correct clean detection. We use CR@0.x to denote the certified recall at a \textit{clean recall} of 0.x.


\begin{figure}[t]
    \centering
    \includegraphics[width=0.9\linewidth]{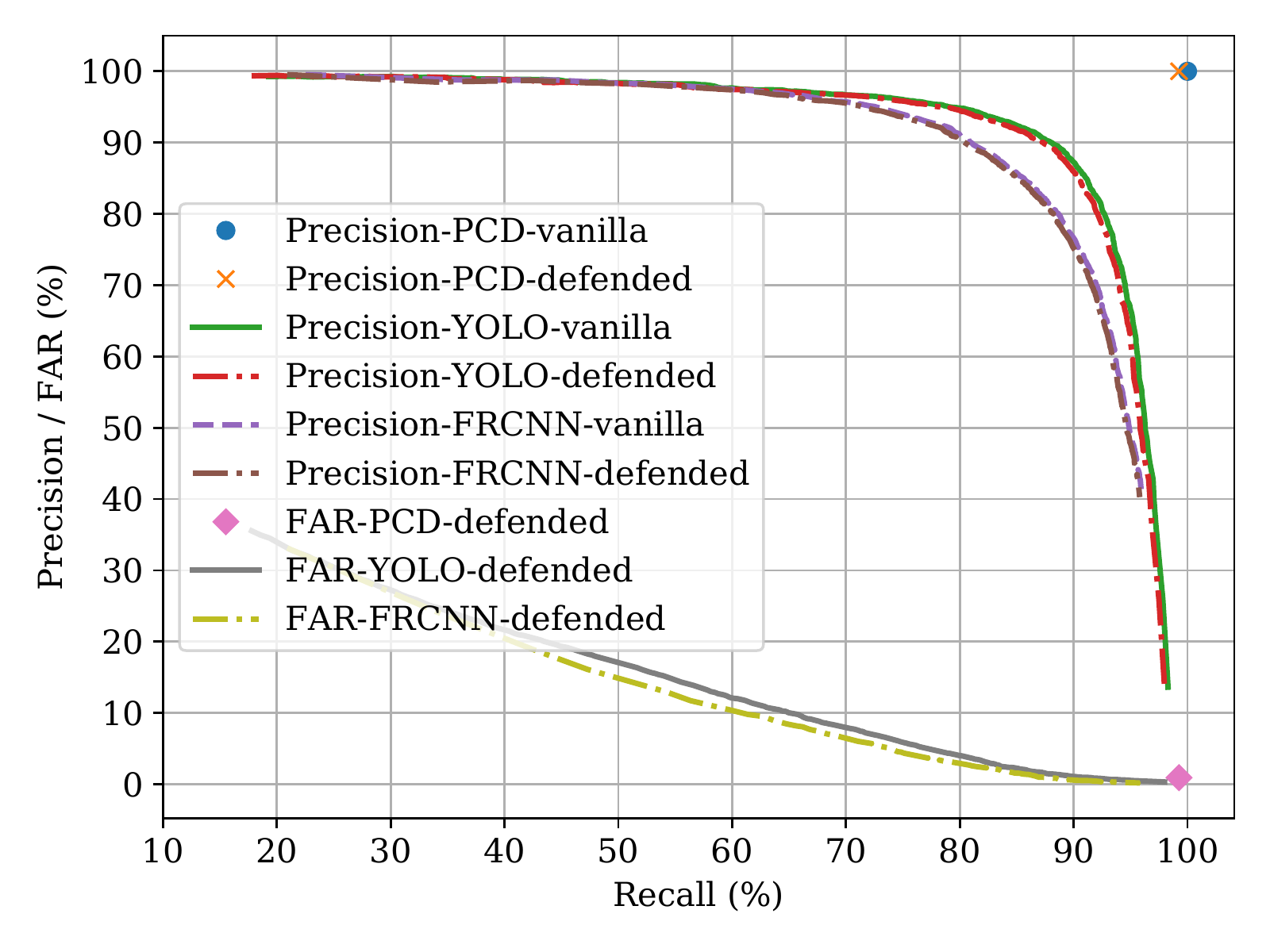}
    \caption{Clean performance of \framework on PASCAL VOC \textmd{(PCD -- perfect clean detector; YOLO -- YOLOv4; FRCNN -- Faster R-CNN; FAR -- False Alert Rate)}}
    \label{fig-clean-voc}
\end{figure}

\begin{table*}[t]
    \centering
    \caption{Provable robustness of \framework}  \label{tab-provable-dpg}
    \vspace{-1em}
    \resizebox{\linewidth}{!}
{ \small\begin{tabular}{c|c|c|c|c|c|c|c|c|c}
    \toprule
    & \multicolumn{3}{c|}{PASCAL VOC (CR@0.8)}&\multicolumn{3}{c|}{MS COCO (CR@0.6)}&\multicolumn{3}{c}{KITTI (CR@0.8)}\\

  & {far-patch}& {close-patch}& {over-patch} &  {far-patch}& {close-patch}& {over-patch} & {far-patch}& {close-patch}& {over-patch}\\
    \midrule
    Perfect clean detector&28.6\%&20.7\%&8.3\% &11.5\%&7.0\%&2.2\%&32.0\%&11.1\%&2.1\%\\
    YOLOv4&24.6\%&18.6\%&7.5\%&10.1\%&6.5\%&1.9\%&30.6\%&10.9\%&2.1\%\\
    Faster R-CNN&25.7\%&19.4\%&8.0\%&10.7\%&6.8\%&2.0\%&31.6\%&11.0\%&2.1\%\\
      \bottomrule
    \end{tabular}}
  
\end{table*}
\subsection{Clean Performance}\label{sec-eval-clean}
In this subsection, we evaluate the clean performance of \framework with three different base object detectors and three datasets. In Table~\ref{tab-clean-dpg}, we report AP of vanilla \base (vanilla AP), AP of \framework (defense AP), and False Alert Rate at a clean recall of 0.8 or 0.6 (FAR@0.8 or FAR@0.6). We also plot the precision-recall and FAR-recall curves for PASCAL VOC in Figure~\ref{fig-clean-voc}; similar plots for MS COCO and KITTI are in Appendix~\ref{apx-exp}.

\textbf{\framework has a low FAR and a high AP.} We can see from Table~\ref{tab-clean-dpg} that \framework has a low FAR of 0.9\% and a high AP of 99.3\% on PASCAL VOC when we use a perfect clean detector as \base. The result shows that \framework only has a minimal impact on the clean performance.

\textbf{\framework is highly compatible with different conventional object detectors.} From Table~\ref{tab-clean-dpg} and Figure~\ref{fig-clean-voc}, we can see that when we use YOLOv4 or Faster R-CNN as \base on PASCAL VOC, the clean AP, as well as the precision-recall curve of \framework, is close to that of its vanilla \base. These results show that \framework is highly compatible with different conventional object detectors.

\textbf{\framework works well across different datasets.} We can see that the observation of high clean performance is similar across three different datasets: \framework achieves a low FAR and a similar AP as the vanilla \base on PASCAL VOC, MS COCO, and KITTI (the precision-recall plots for MS COCO and KITTI are available in Appendix~\ref{apx-exp}). These similar results show that \framework is a general approach and can be used for both easier and more challenging detection tasks.

\textbf{Remark: a negligible cost of clean performance.} In this subsection, we have shown that \framework only incurs a negligible cost of clean performance (<1\% AP drop). This slight clean performance drop is worthwhile given the first provable robustness against patch hiding attacks (evaluated in the next subsection).

\subsection{Provable Robustness}\label{sec-eval-provable}
In this subsection, we first introduce the robustness evaluation setup and then report the provable robustness of our defense against any patch hiding attack within our threat model.

\textbf{Setup.} We use a 32$\times$32 adversarial pixel patch on the re-scaled and padded 416$\times$416 (or 224$\times$740) images to evaluate the provable robustness.\footnote{DPatch~\cite{liu2019dpatch} demonstrates that even a 20$\times$20 adversarial patch at the image corner can have a malicious effect. In Appendix~\ref{apx-obj-size}, we show that more than 15\% of PASCAL VOC objects and 44\% of MS COCO objects are smaller than a 32$\times$32 patch. We also provide robustness results for different patch sizes as well as visualizations in Appendix~\ref{apx-obj-size}.}
We consider \textit{all possible image locations} as candidate locations for the adversarial patch to evaluate the model robustness. We categorize our results into three categories depending on the distance between an object and the patch location. 
When the patch is totally over the object, we consider it as \emph{over-patch}. When the feature-space distance between patch and object is smaller than 8, we consider it as  \emph{close-patch}. The other patch locations are considered as \emph{far-patch}. For each set of patch locations and each object, we use Algorithm~\ref{alg-dpg-analysis} to determine the robustness. 
We note that the above algorithm already considers \emph{all possible adaptive attacks} (attacker strategies) at any location within our threat model. 
We use CR@0.x as the robustness metric. Given the large number of all possible patch locations, we only use a 500-image subset of the test/validation datasets for evaluation.

\begin{figure}[t]
    \centering
    \includegraphics[width=0.9\linewidth]{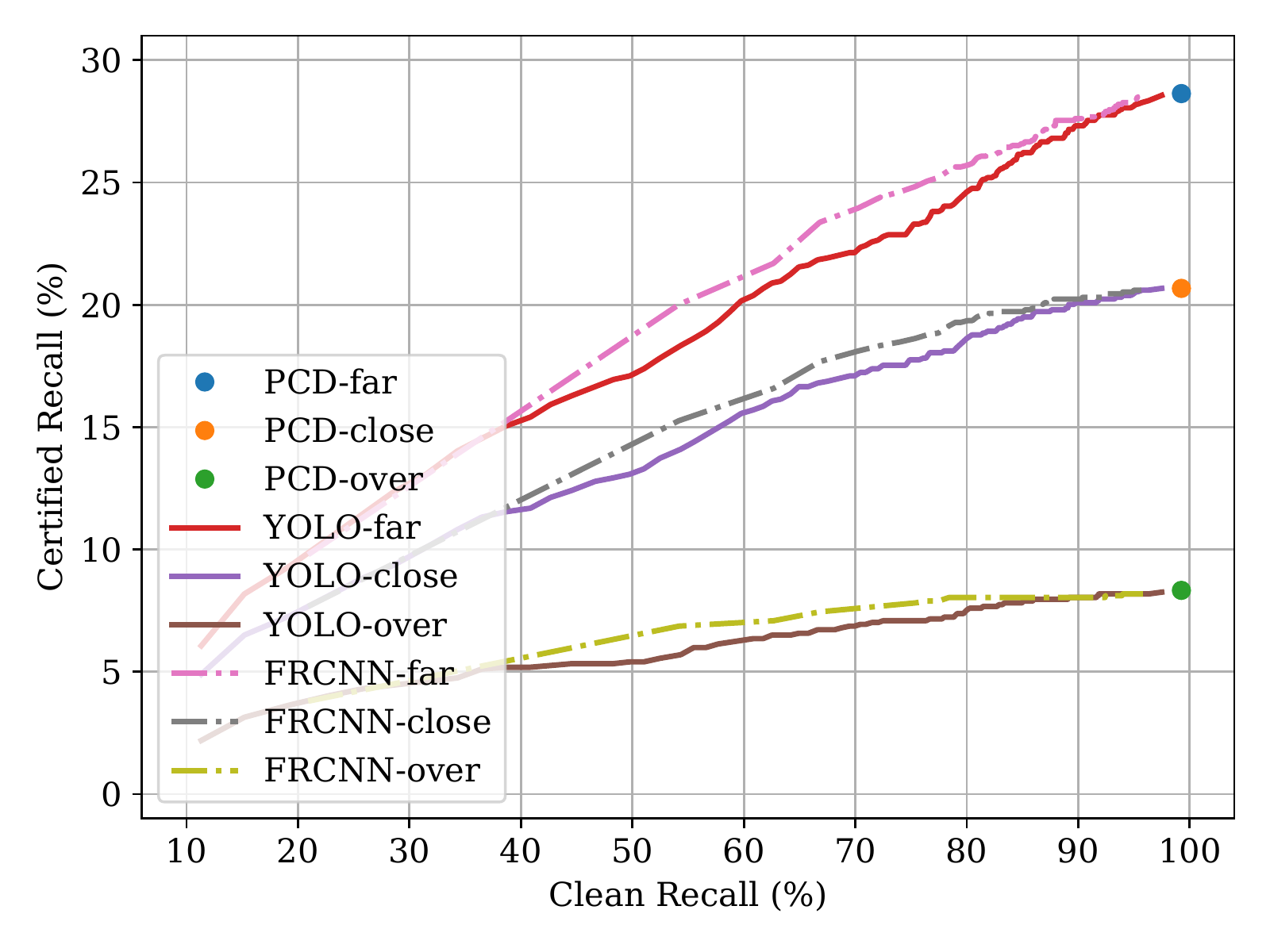}
    \caption{Provable robustness of \framework on PASCAL VOC \textmd{(PCD -- Perfect Clean Detector; YOLO -- YOLOv4; FRCNN -- Faster R-CNN) }}
    \label{fig-pro-voc}
\end{figure}

\textbf{\framework achieves the first non-trivial provable robustness against patch hiding attack.} We report the certified recall at a clean recall of 0.8 or 0.6 (CR@0.8 or CR@0.6) in Table~\ref{tab-provable-dpg}. As shown in Table~\ref{tab-provable-dpg}, when we use a Perfect Clean Detector, \framework can certify the robustness for 28.6\% of PASCAL VOC objects when the patch is far away from the object; which means \textit{no attacker} within our threat model can successfully attack these certified objects. We also plot the CR-recall curves for different detectors on the PASCAL VOC dataset in Figure~\ref{fig-pro-voc} (similar plots for MS COCO and KITTI are in Appendix~\ref{apx-exp}). The figure shows that the provable robustness improves as the clean recall increases, and the performance of YOLOv4 and Faster R-CNN is close to that of a perfect clean detector when the recall is close to one.

\textbf{\framework is especially effective when the patch is far away from the objects.}
From Table~\ref{tab-provable-dpg} and Figure~\ref{fig-pro-voc}, we can clearly see that the provable robustness of \framework is especially good when the patch gets far away from the object. This model behavior aligns with our intuition that a localized adversarial patch should only have a spatially constrained adversarial effect. Moreover, this observation shows that \framework has made the attack much more difficult: to have a chance to bypass \framework, the adversary has to put the patch close to or even over the victim object, which is not always feasible in real-world scenarios. We also note that in the over-patch threat model, we allow the patch to be anywhere over the object. This means that the patch can be placed over the most salient part of the object (e.g., the face of a person), and makes robust detection extremely difficult.

\textbf{\framework has better robustness for object classes with larger object sizes.} In Figure~\ref{fig-per-cls-voc-close}, we plot the certified recalls against the close-patch attacker (similar plots for far-patch and over-patch and for other datasets are in Appendix~\ref{apx-exp}) and average object sizes (reported as the percentage of image pixels) for all 20 classes of PASCAL VOC. As shown in the figure, the provable robustness of \framework is highly correlated with the object size: we have higher certified recalls for classes with larger object sizes. This is an expected behavior because it is hard for even humans to perfectly detect all small objects. Moreover, considering that missing a big nearby object is much more serious than missing a small distant object in real-world applications, we believe that \framework has strong application potential.

\textbf{Remark on absolute values of certified recalls.} Despite the first achieved provable robustness against patch hiding attacks, we acknowledge that DetectorGuard's absolute values of certified recalls are still limited. First, the notion of the certified recall itself is strong and conservative: we certify the robustness of an object only when no patch at any location within the threat model can succeed in the hiding attack (using any attack strategy including adaptive attacks). In practice, the attacker capability might be limited to a small number of patch locations. Second, we note that most objects in our three datasets are small (or even tiny) objects (we provide a quantitative analysis in Appendix~\ref{apx-obj-size}). Detecting small objects is already challenging in the clean setting and becomes even more difficult when an adversarial patch of comparable sizes is used. However, it is still notable that \framework is the first to achieve non-trivial and strong provable robustness against the patch hiding attacker for ``free", and that our approach works particularly well on certain object classes and threat models (e.g., the class ``dog" in the PASCAL VOC dataset has a \textasciitilde60\% certified recall for the close-patch threat model). We hope that the security community can build upon our work and further push the certified recall to a higher value.

\subsection{Detailed Analysis of \framework}\label{sec-eval-hyp}

In this subsection, we first perform a runtime analysis of \framework to show its small defense overhead. Next, we use a hypothetical perfect clean detector (PCD) and the PASCAL VOC dataset to analyze the performance of \framework under different hyper-parameter settings. Note that using PCD helps us to focus on the behavior of \defense and \matcher.

\textbf{Runtime analysis.} In Table~\ref{tab-runtime}, we report the average per-example runtime of each module in \framework. For \base, we report runtime for YOLOv4 (left) and Faster R-CNN (right). As shown in the table, \defense has a similar runtime as \base (or vanilla undefended object detectors), and \matcher only introduces a negligible overhead. If 2 GPUs are available, then  \base and \defense can run in parallel in \framework, and the overall runtime of \framework can be calculated as $t_{\text{\framework}} = \max(t_{\text{base}},t_{\text{predictor}})+t_{\text{explainer}}$ (reported in in the last column of Table~\ref{tab-runtime}), which is close to $t_{\text{base}}$.
Thus, \framework has a similar runtime performance as conventional object detectors in the setting of 2 GPUs. If only a single GPU is available, the runtime can be calculated as $t_{\text{base}}+t_{\text{predictor}}+t_{\text{explainer}}$, which leads to a roughly 2x slow-down.

\begin{figure}[t]
\centering
\includegraphics[width=\linewidth]{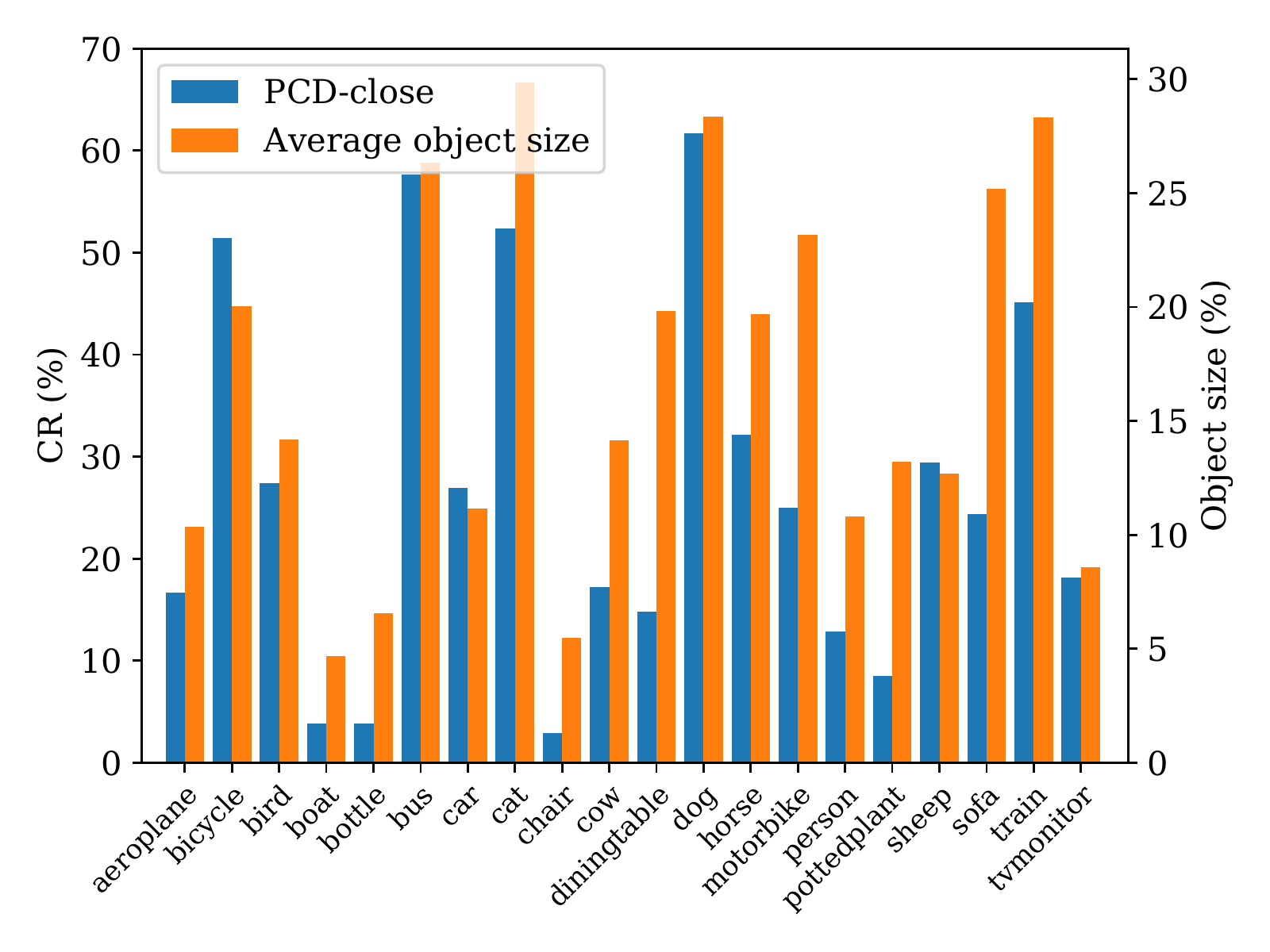}\\
\caption{Certified recalls and average object sizes for every class in PASCAL VOC \textmd{(reporting CR for close-patch; results for far-patch and over-patch are in Appendix~\ref{apx-exp})}}\label{fig-per-cls-voc-close}
\end{figure}

\begin{figure*}
\begin{minipage}[b]{0.25\linewidth}
\includegraphics[width=\linewidth]{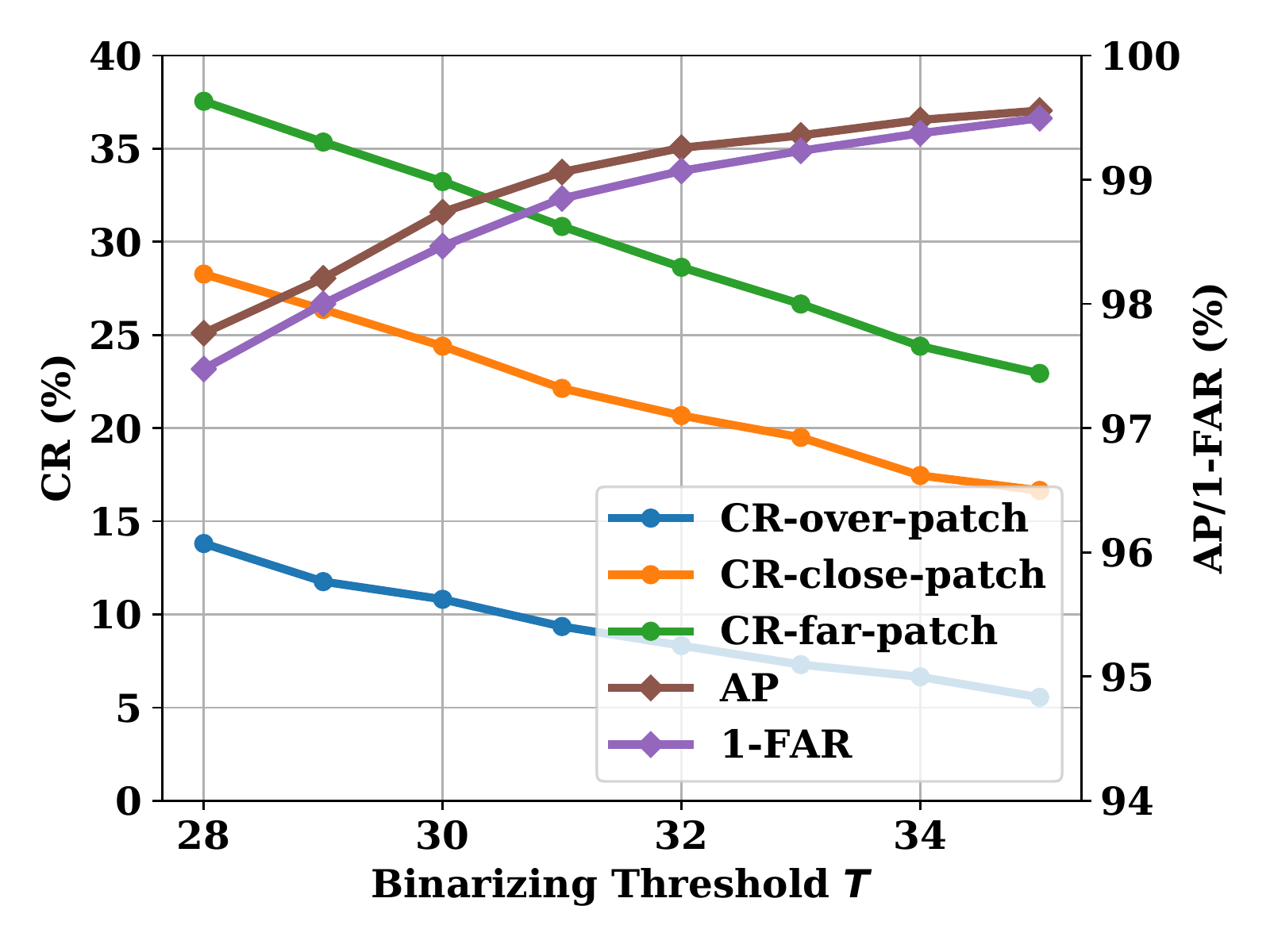}\\
\end{minipage}%
\begin{minipage}[b]{0.25\linewidth}
\includegraphics[width=\linewidth]{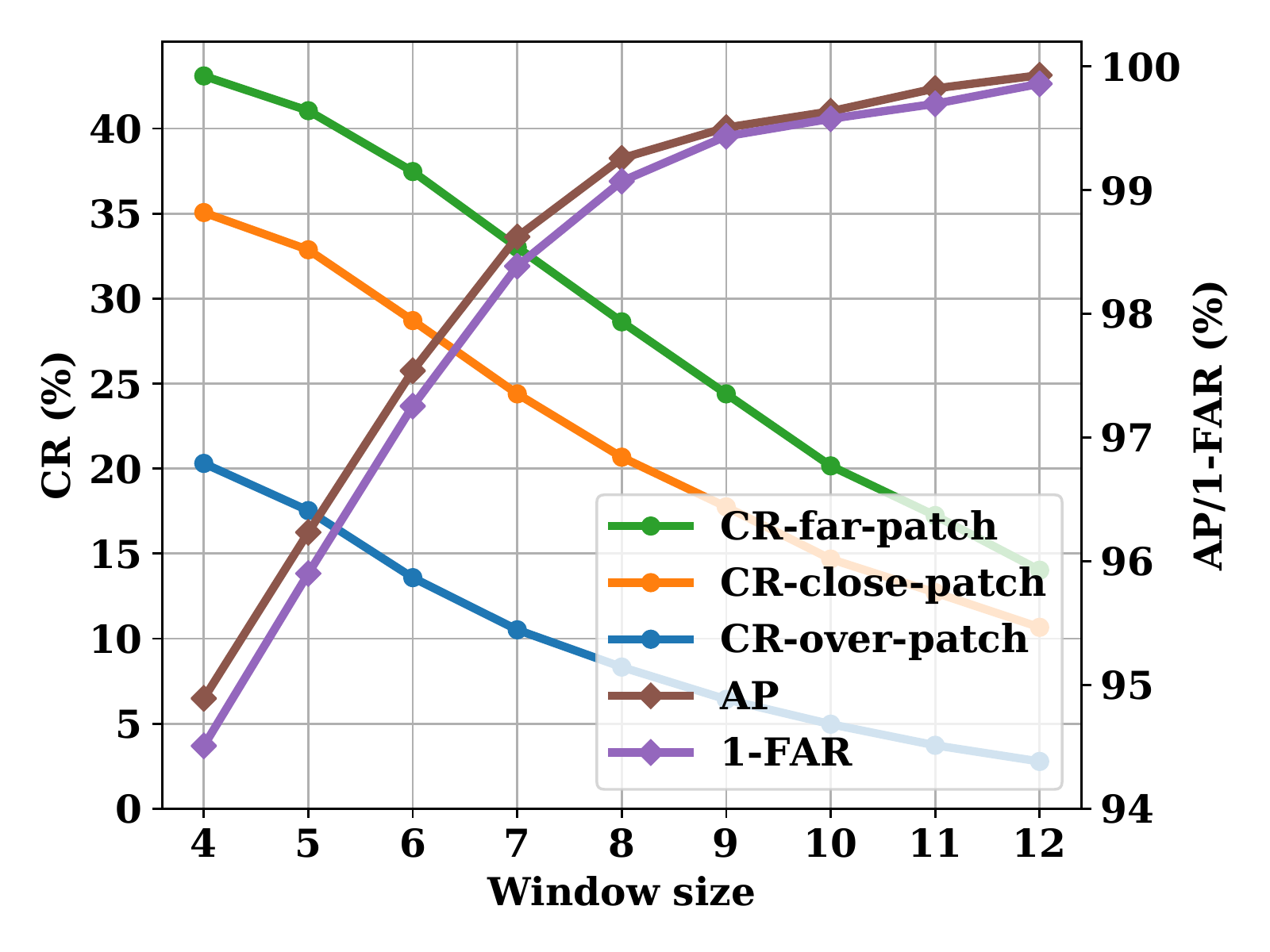}\\
\end{minipage}%
\begin{minipage}[b]{0.25\linewidth}
\includegraphics[width=\linewidth]{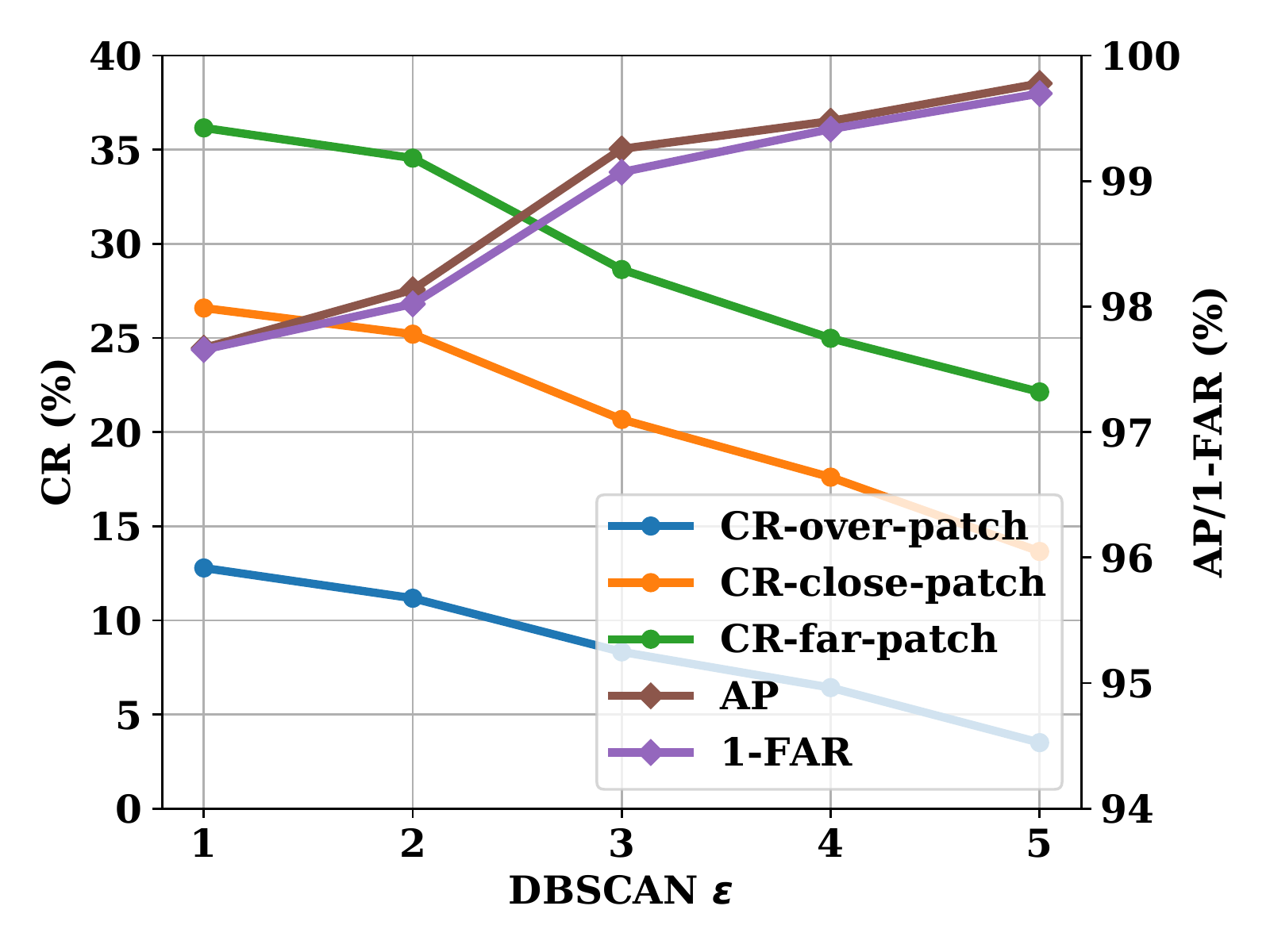}\\
\end{minipage}%
\begin{minipage}[b]{0.25\linewidth}
\includegraphics[width=\linewidth]{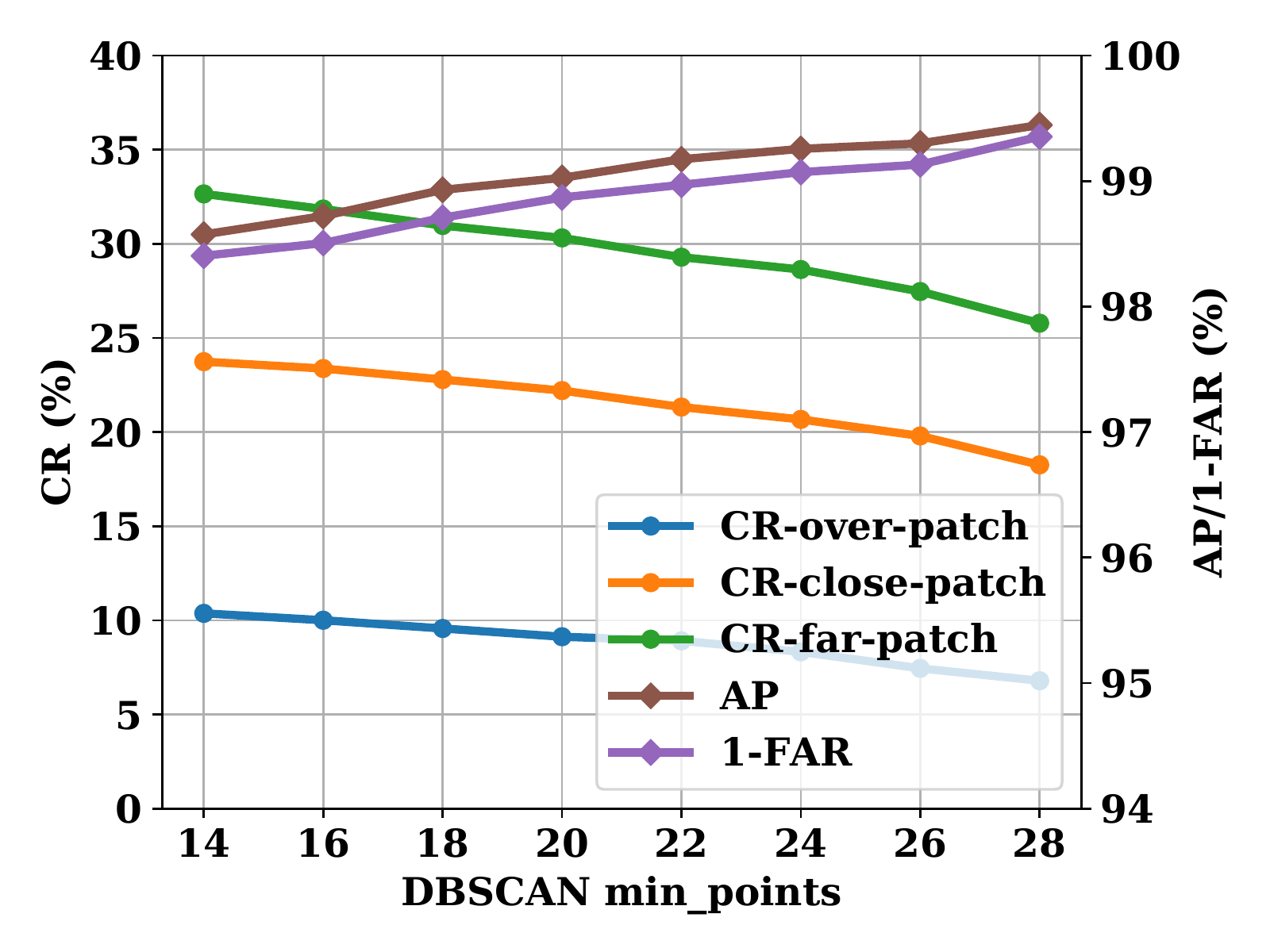}\\
\end{minipage}%
\caption{Effect of different hyper-parameters \textmd{(left to right: binarizing threshold, window size, DBSCAN $\epsilon$, DBSCAN \texttt{min\_points})}}\label{fig-hyp}
\end{figure*}

\begin{table}[t]
    \centering
    \caption{Per-example runtime breakdown}\label{tab-runtime}
    \vspace{-1em}
    \resizebox{\linewidth}{!}{
    \begin{tabular}{c|c|c|c|c}
    \toprule
        &\base &Objectness &Objectness &\framework\\
        
         &(YOLO / FRCNN)&Predictor &Explainer&(YOLO / FRCNN)\\
         \midrule
    VOC     & 54.0ms / 80.9ms & 48.5ms&0.2ms&54.2ms / 81.1ms\\
    COCO & 55.2ms / 79.4ms&65.2ms&0.3ms&65.5ms / 79.7ms\\
    KITTI &55.8ms / 69.7ms &44.6ms&0.4ms& 56.2ms / 70.1ms\\
    \bottomrule
    \end{tabular}}
\end{table}

\textbf{Effect of the binarizing threshold.} We first vary the binarizing threshold $T$ in $\textsc{ObjPredictor}(\cdot)$ to see how the model performance changes. For each threshold, we report CR for three patch threat models.  We also include AP and 1-FAR to understand the effect of different thresholds on clean performance. We report these results in the leftmost sub-figure in Figure~\ref{fig-hyp}. We can see that when the binarizing threshold is low, the CR is high because more objectness is retained after the binarization. However, more objectness also makes it more likely to trigger a false alert in the clean setting, and we can see both AP and 1-FAR are affected greatly as we decrease the threshold $T$. Therefore, we need to balance the trade-off between clean performance and provable robustness. In our default parameter setting, we set $T=32$ to have a FAR lower than 1\% while maintaining decent provable robustness.

\textbf{Effect of window size.} We study the effect of different window sizes in the second sub-figure in Figure~\ref{fig-hyp}. The figure demonstrates a similar trade-off between provable robustness and clean performance. As we increase the window size, each window receives more information from the input and therefore the clean performance (AP and 1-FAR) improves. However, a large window size increases the number of windows that are affected by the small adversarial patch, and the provable robustness drops. In our default setting, we set the window size to 8 to have a low FAR and good CR.

\textbf{Effect of DBSCAN parameters.} We also analyze the effect of DBSCAN parameters in $\textsc{DetCluster}(\cdot)$. DBSCAN has two parameters $\epsilon$ and \texttt{min\_points}. A point is labeled as a core point when there are at least \texttt{min\_points} points within distance $\epsilon$ of it; all core points and their neighbors will be labeled as clusters. We plot the effect of $\epsilon$ and \texttt{min\_points} in the right two sub-figures in Figure~\ref{fig-hyp}. As we increase $\epsilon$ or \texttt{min\_points}, it becomes more difficult to form clusters. As a result, the clean performance improves because of fewer detected clusters and fewer false alerts. However, the provable robustness (CR) drops due to fewer detected clusters in the worst-case objectness map. 

\section{Discussion}\label{sec-discussion}
In this section, we discuss the future work directions and defense extension of \framework.


\subsection{Future Work}
\textbf{Robust object detection without abstention.} In this paper, we have tailored \framework for attack detection: when no attack is detected, the model uses conventional object detectors for predictions; when an attack is detected, the model alerts and abstains from making predictions. This type of defense is useful in application scenarios like autonomous vehicles which can give the control back to the driver upon detecting an attack. However, the most desirable notion of robustness is to always make correct predictions without any abstention. How to extend \framework for robust object detection without abstention is an interesting direction of future work.

\textbf{Better robust image classifier.} In \framework, we use the key principals introduced in Section~\ref{sec-prelim-classifier} to co-design the provably robust image classifier and \defense. However, the imperfection of the adapted robust image classifier still limits the performance of \framework detector. Although we can optimize for few \textit{Clean Error 2} and tolerate a potentially high \textit{Clean Error 3}, as discussed in Section~\ref{sec-matcher}, a high \textit{Clean Error 3} results in a limited certified recall of \framework in the adversarial setting. We note that \framework is a general framework that is \textit{compatible with any conventional object detector and (principles for building) provably robust image classifier}, and we expect a boost in \framework's performance given any future advance in robust image classification research.

\textbf{Extension to the video setting.} In this paper, we focus on object detection in the single-frame setting. It is interesting to extend our \framework to multiple-frame video setting. We expect that the temporal information could be useful for robustness. Moreover, we could perform the defense on a subset of frames to reduce defense overhead and minimize false alerts in the clean setting.

\begin{algorithm}[t]
\caption{Auxiliary Predictor of \framework}\label{alg-aux}
\begin{algorithmic}[1]
\Procedure{AuxPredictor}{$\mathcal{D},\mathbf{x}$}
\State $\hat{\mathcal{D}} \gets \{\}$
\For{$i\in\{0,1,\cdots,num\_detection-1\}$}
\State $x_{\min},y_{\min},x_{\max},y_{\max},l \gets \mathbf{b} \gets \mathcal{D}[i]$
\State $l^\prime \gets \textsc{AuxClassier}(\mathbf{x}[x_{\min}:x_{\max},y_{\min}:y_{\max}])$\label{ln-reclassify}
\If{$l^\prime\neq\text{``background"}$}
\State $\hat{\mathcal{D}} \gets \hat{\mathcal{D}}\bigcup \{(x_{\min},y_{\min},x_{\max},y_{\max},l^\prime)\}$\label{ln-retain}
\EndIf
\EndFor
\State\Return $\hat{\mathcal{D}}$
\EndProcedure
\end{algorithmic}
\end{algorithm}
\subsection{Defense Extension.} 
In this paper, we propose \framework as a provably robust defense against \textit{patch hiding, or false-negative (FN), attacks}. Here, we discuss how to extend \framework for defending against \textit{false-positive (FP)} attacks. The FP attack aims to introduce incorrect bounding boxes in the predictions of detectors to increase FP. We can consider FP attacks as a misclassification problem (i.e., a bounding box is given an incorrect label), and thus this attack can be mitigated if we have a robust auxiliary image classifier to re-classify the detected bounding boxes. If the auxiliary classifier predicts a different label, we consider it as an FP attack and can easily correct or filter out the FP boxes. 

We provide the pseudocode for using an auxiliary classifier (can be any robust image classifier) against FP attacks in Algorithm~\ref{alg-aux}. The algorithm re-classifies each detected bounding box in $\mathcal{D}$ as label $l^\prime$ (Line~\ref{ln-reclassify}). We trust the robust auxiliary classifier and add bounding boxes with non-background labels to $\hat{\mathcal{D}}$  (Line~\ref{ln-retain}). Finally, the algorithm returns the filtered detection $\hat{\mathcal{D}}$, and we can replace the $\mathcal{D}$ in Line~\ref{ln-return} of Algorithm~\ref{alg-dpg} with $\hat{\mathcal{D}}$ to extend the original \framework design. We note that when the patch is not present in the FP box or it only occupies a small portion of the FP box (the FP box is large), the auxiliary classifier is likely to correctly predict ``background" since there is no or few corrupted pixel within the FP box. When the patch occupies a large portion of the FP box, we might finally output a bounding box that has a large IoU with the patch. In this case, our object detector correctly locates the adversarial patch but predicts a wrong label. This is an acceptable outcome because the class label for patches is undefined. Therefore, Algorithm~\ref{alg-aux} presents a strong empirical defense for FP attacks.

\section{Related Work}\label{sec-related-work}
\subsection{Adversarial Patch Attacks}
\textbf{Image Classification.}
Unlike most adversarial examples that introduce a global perturbation with a $L_p$-norm constraint, localized adversarial patch attacks allow the attacker to introduce arbitrary perturbations within a restricted region. Brown et al.~\cite{brown2017adversarial} introduced the first adversarial patch attack against image classifiers. They successfully realized a real-world attack by attaching a patch to the victim object. Follow-up papers studied variants of localized attacks against image classifiers with different threat models~\cite{karmon2018lavan,liu2019perceptual,liu2020bias}.

\textbf{Object Detection.}
Localized patch attacks against object detection have also received much attention in the past few years. Liu et al.~\cite{liu2019dpatch} proposed DPatch as the first patch attack against object detectors in the digital domain. 
Lu et al.~\cite{lu2017adversarial}, Chen et al.~\cite{chen2018shapeshifter}, Eykholt et al.~\cite{eykholt2018physical}, and Zhao et al.~\cite{zhao2019seeing} proposed different physical attacks against object detectors for traffic sign recognition. Thys et al.~\cite{thys2019fooling} proposed to use a rigid physical patch to evade human detection while Xu et al.~\cite{xu2020adversarial} and Wu et al.~\cite{wu2019making} generated successfully non-rigid perturbations on T-shirt to evade detection. 

\subsection{Defenses against Adversarial Patches}
\noindent\textbf{Image Classification.} Digital Watermark (DW)~\cite{hayes2018visible} and Local Gradient Smoothing (LGS)~\cite{naseer2019local} were the first two heuristic defenses against adversarial patch attacks. Unfortunately, these defenses are vulnerable to an adaptive attacker with the knowledge of the defense. A few certified defenses~\cite{chiang2020certified,zhang2020clipped,levine2020randomized,mccoyd2020minority,xiang2020patchguard,metzen2021efficient,xiang2021patchguardpp} have been proposed to provide strong provable robustness guarantee against any adaptive attacker. Notably, PatchGuard~\cite{xiang2020patchguard} introduces two key principles of small receptive fields and secure aggregation and achieves state-of-the-art defense performance for image classification. In contrast, \framework aims to adapt robust image classifiers for the more challenging robust object detection task.  

\textbf{Object Detection.} How to secure object detection is a much less studied area due to the complexity of this task. Saha et al.~\cite{saha2020role} demonstrated that YOLOv2~\cite{redmon2017yolo9000} were vulnerable to adversarial patches because detectors were using spatial context for their predictions, and then proposed a new training loss to limit the usage of context information. To the best of our knowledge, this is the only prior attempt to secure object detectors from patch attacks. However, this defense is based on heuristics and thus does not have any provable robustness. Moreover, the attack and defense are targeted at YOLOv2 only, and it is unclear if the defense generalizes to other detectors. In contrast, our defense has provable robustness against any patch hiding attack considered in our threat model and is compatible with any state-of-the-art object detectors.

\subsection{Other Adversarial Example Attacks and Defenses}

\textbf{Image Classification.} Attacks and defenses for classic $L_p$-bounded adversarial examples~\cite{szegedy2013intriguing,goodfellow2014explaining,carlini2017towards} have been extensively studied. Many empirical defenses~\cite{papernot2016distillation,xu2017feature,meng2017magnet,metzen2017detecting,madry2017towards} were proposed to mitigate the threat of adversarial examples, but were later found vulnerable to adaptive attackers~\cite{carlini2017adversarial,athalye2018obfuscated,tramer2020adaptive}. The fragility of the empirical defenses has inspired certified defenses that are robust to any attacker considered in the threat model~\cite{raghunathan2018certified,wong2017provable,lecuyer2019certified,cohen2019certified,salman2019provably,gowal2018effectiveness,mirman2018differentiable}. 
We refer interested readers to survey papers~\cite{papernot2018sok,yuan2019adversarial} for more details.

\textbf{Object Detection.} Global perturbations against object detectors were first studied by Xie et al.~\cite{xie2017adversarial} and followed by researchers~\cite{wang2019daedalus,wei2019transferable} in different applications. Defenses against global $L_p$ perturbations are also very challenging. Zhang et al.~\cite{zhang2019towards} used adversarial training (AT) to improve empirical model robustness while Chiang et al.~\cite{chiang2020detection} proposed the use of randomized median smoothing (RMS) for building certifiably robust object detectors. Both defenses suffer from poor clean performance while \framework's clean performance is close to state-of-the-art object detectors. On PASCAL VOC, AT incurs a \textasciitilde26\% clean AP drop while \framework only incurs a <1\% drop.\footnote{RMS~\cite{chiang2020detection} did not report results for PASCAL VOC.} On MS COCO, both AT and RMS have a clean AP drop that is larger than 10\% while ours is smaller than 1\%. We note that we do not compare robustness performance because these two works focus on global perturbations and are orthogonal to the objective of this paper. 

We note that it is possible to extend our robust objectness predictor design and objectness explaining strategy to mitigate attacks that use global perturbations with a bounded $L_\infty$ norm (if we have a robust image classifier against $L_\infty$ perturbations). We leave this as a future work direction.

\section{Conclusion}
In this paper, we propose \framework, the first general framework for building provably robust object detectors against patch hiding attacks. \framework introduces a general approach to adapt robust image classifiers for robust object detection using an objectness explaining strategy. Our evaluation on the PASCAL VOC, MS COCO, and KITTI datasets demonstrates that \framework achieves the first provable robustness against any patch hiding attacker within the threat model and also has a high clean performance that is close to state-of-the-art detectors.

\begin{acks}
We are grateful to Gagandeep Singh for shepherding the paper and anonymous reviewers at CCS 2021 for their valuable feedback. We would also like to thank Vikash Sehwag, Shawn Shan, Sihui Dai, Alexander Valtchanov, Ruiheng Chang, Jiachen Sun, and researchers at Intel Labs for helpful discussions on the project and insightful comments on the paper draft. This work was supported in part by the National Science Foundation under grants CNS-1553437 and CNS-1704105, the ARL’s Army Artificial Intelligence Innovation Institute (A2I2), the Office of Naval Research Young Investigator Award, Schmidt DataX award, and Princeton E-ffiliates Award.

\end{acks}

\bibliographystyle{ACM-Reference-Format}
\bibliography{bib}
\newpage

\appendix

\begin{figure*}[]
\centering
\begin{minipage}[b]{0.3\linewidth}
\includegraphics[width=\linewidth]{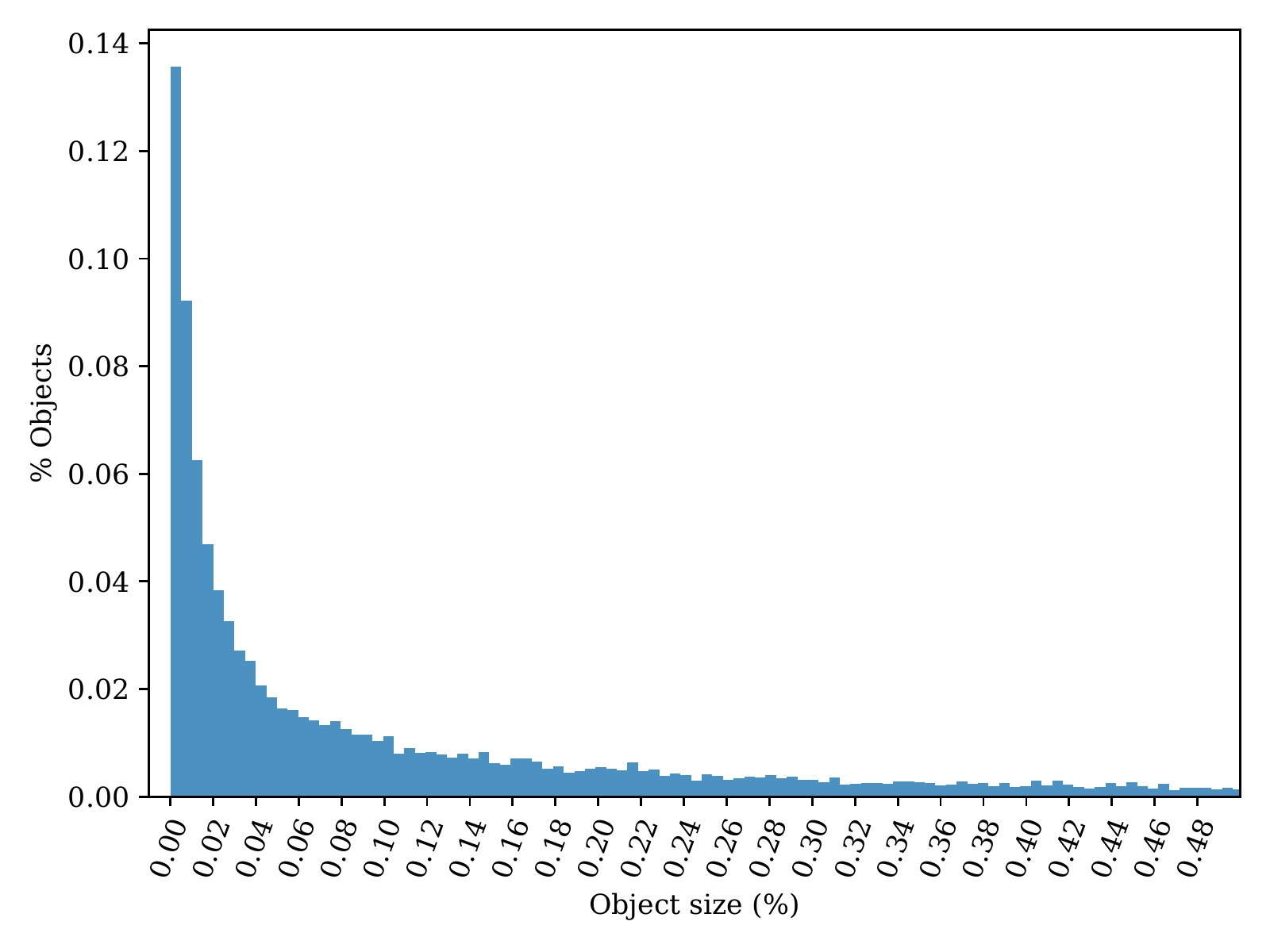}\\
\end{minipage}%
\begin{minipage}[b]{0.3\linewidth}
\includegraphics[width=\linewidth]{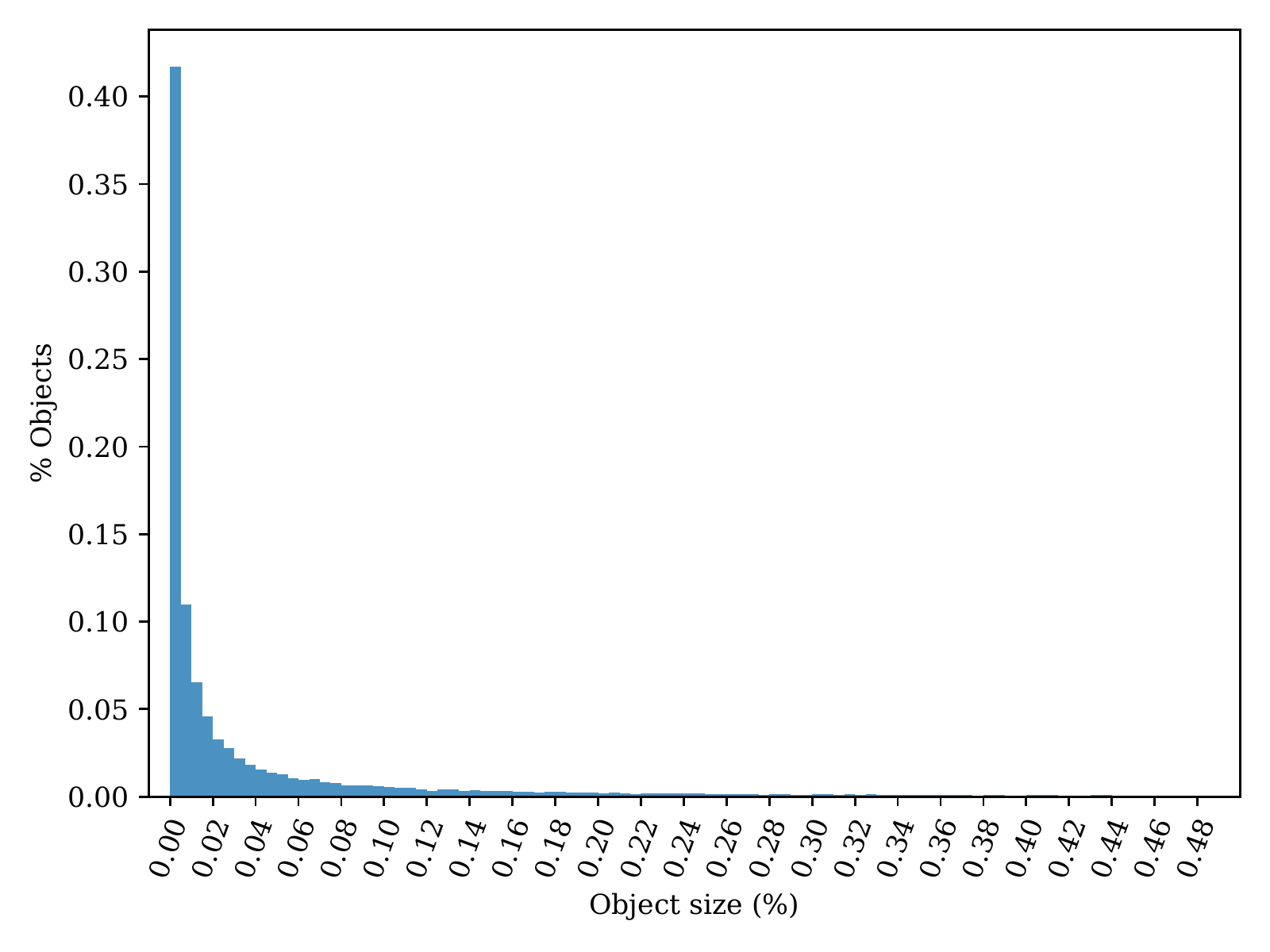}\\
\end{minipage}%
\begin{minipage}[b]{0.3\linewidth}
\includegraphics[width=\linewidth]{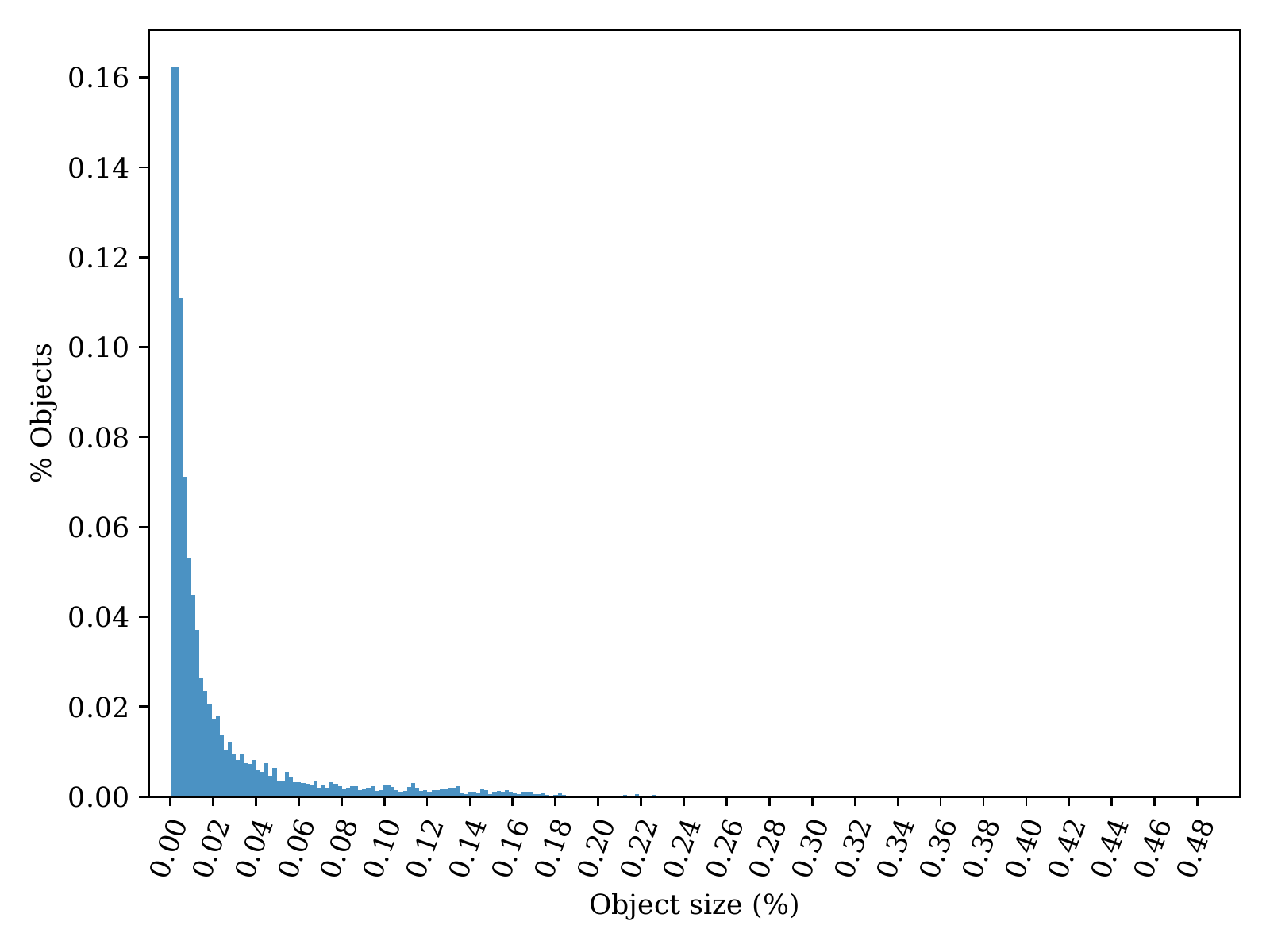}\\
\end{minipage}%
\caption{Histogram of object sizes (left: PASCAL VOC; middle: MS COCO; right: KITTI)}\label{fig-hist-obj-size}
\end{figure*}

\section{Object Size and Patch Size}\label{apx-obj-size}
Recall that in Section~\ref{sec-eval-provable}, we use a 32$\times$32 patch on 416$\times$416 (or 224$\times$740) images to evaluate the provable robustness. In this section, we provide additional details of object sizes and patch sizes in PASCAL VOC, MS COCO, and KITTI datasets.

\noindent \textbf{Small objects are the majority of all three datasets.} In Figure~\ref{fig-hist-obj-size}, we plot the histogram of object sizes (in the percentage of pixels) in the test/validation sets of PASCAL VOC, MS COCO, and KITTI. As shown in the plots, small, or even tiny, objects are the majority of three datasets. A 32$\times$32 patch takes up 0.6\% pixels of a 416$\times$416 (or 224$\times$740) image, and our further analysis shows that 15.2\% objects of PASCAL VOC are smaller than 0.6\% image pixels; 44.5\% of MS COCO and 44.6\% KITTI objects are smaller than 0.6\%. Moreover, more than 36.5\% of PASCAL VOC objects, more than 66.3\% of MS COCO objects, and 75.9\% KITTI objects are smaller than a 64$\times$64 square. These numbers explain why the absolute numbers of certified recall in Table~\ref{tab-provable-dpg} are low. In Figure~\ref{fig-patch-obj-visu}, we further provide visualization of a 32$\times$32 patch on the 416$\times$416 image to demonstrate the challenge of perfect robust detection even when a small patch is presented. In the left two examples, the person and the cow are completely blocked by the adversarial patch and thus are unrecognizable. In the rightmost example, the head of the dog is patched and it is even hard for humans to determine if it is a dog or cat.

\noindent\textbf{Additional evaluation results for different patch sizes.} In Figure~\ref{fig-p-size}, we vary the patch size to see how the provable robustness is affected given different attacker capabilities (i.e, patch sizes). If we consider a smaller patch of 8$\times$8 pixels, we can have a 2.0\% higher CR for close-patch, and a 2.8\% higher CR for over-patch compared with our CRs for a 32$\times$32 patch. Furthermore, we note that in the far-patch model, the patch size has a limited influence on the provable robustness. From Figure~\ref{fig-p-size}, We can also see that the CR decreases as the patch size increases. This analysis demonstrates the limit of \framework as well the challenge of robust object detection with larger patch sizes. We aim to push this limit further in our future work.

\begin{figure}[t]
    \centering
    \includegraphics[width=\linewidth]{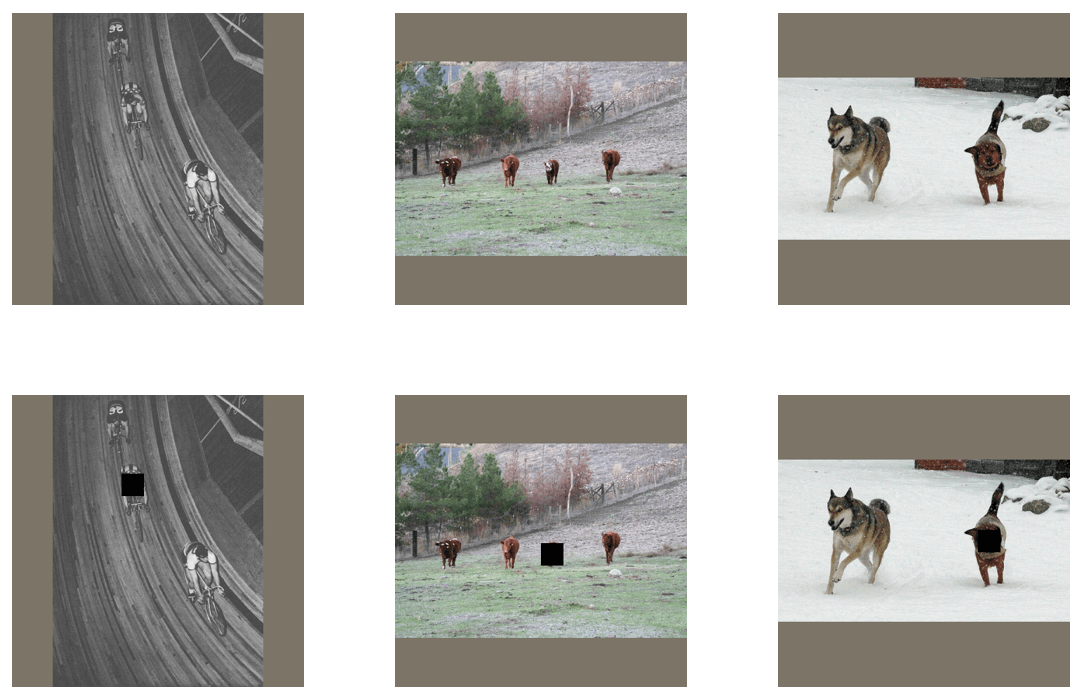}
    \caption{Visualization of patches on small objects (upper: original 416$\times$416 images; lower: images with a 32$\times$32 black patch)}
    \label{fig-patch-obj-visu}
\end{figure}

\begin{figure}[t]
    \centering
    \includegraphics[width=0.8\linewidth]{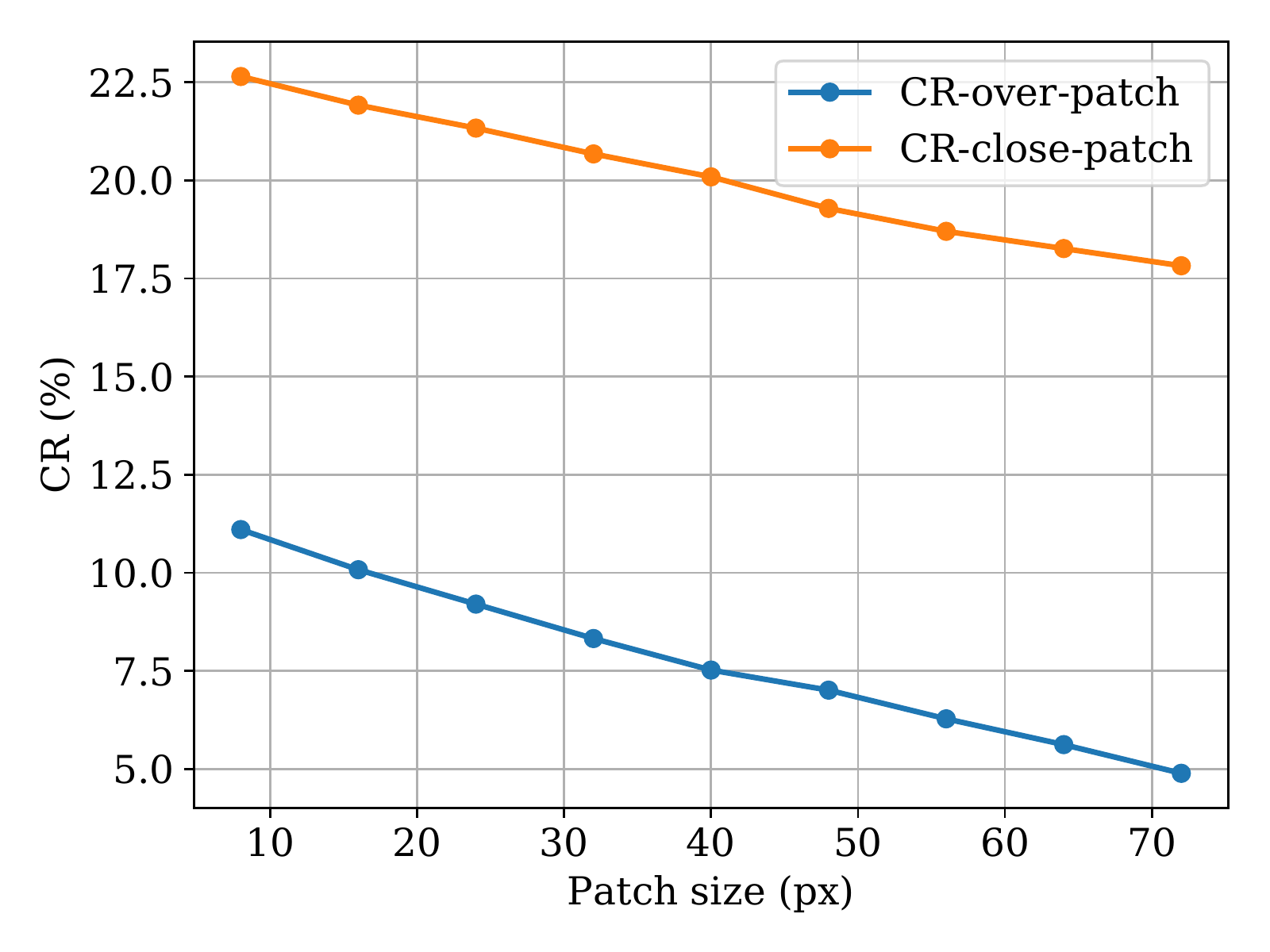}
    \caption{Effect of patch size on provable robustness of \framework with a perfect clean detector}
    \label{fig-p-size}
\end{figure}

\begin{table*}[t]
    \centering
    \caption{Comparison between masking-based and clipping-based defenses of \framework (using a perfect clean detector)}
    { \begin{tabular}{c|c|c|c|c|c|c|c|c|c|c}
    \toprule
    & \multicolumn{5}{c|}{clipping-based \framework}&\multicolumn{5}{c}{masking-based \framework}\\
    & AP &FAR & CR-far &CR-close & CR-over &AP  &FAR & CR-far &CR-close & CR-over \\
    \midrule
    PASCAL VOC & 99.3\%&0.9\%&28.6\%&20.7\%&8.3\%&98.9\% &1.1\%&26.2\%&17.2\%&4.7\%\\
    
    MS COCO& 99.0\% &1.2\% &11.5\%&7.0\%&2.2\% &98.7\%&1.4\%&11.4\%&5.4\%&1.6\%\\
    KITTI&99.0\%&1.5\%&31.6\%&11.0\%&2.1\% &99.4\%&1.1\%&17.4\%&4.9\%&1.2\%\\

      \bottomrule
    \end{tabular}}
    \label{tab-clipping-dpg}
\end{table*}

\section{Additional Discussion of Robust Classifier Implementation}\label{apx-robust-classifier}

As introduced in Section~\ref{sec-prelim-classifier}, state-of-the-art provable robust image classifiers~\cite{xiang2020patchguard,metzen2021efficient} 1) use DNN with small receptive fields to bound the number of corrupted features and then 2) do secure aggregation on the partially corrupted feature map for robust classification. In \framework, we choose BagNet~\cite{brendel2019approximating} for small receptive fields and feature clipping for secure aggregation. In this section, we provide additional details of BagNet and clipping aggregation. Furthermore, we discuss alternative aggregation mechanisms and implement robust masking~\cite{xiang2020patchguard} to demonstrate the generality of our \framework framework.

\noindent\textbf{BagNet~\cite{brendel2019approximating}.} BagNet was originally proposed for interpretable machine learning. It inherits the high-level architecture of ResNet-50~\cite{he2016deep} and replaces 3x3 convolution kernels with 1x1 ones to reduce the receptive field size. The authors designed three BagNet architectures with a small receptive field of 9$\times$9, 17$\times$17, and 33$\times$33, in contrast to ResNet-50 having a receptive field of 483$\times$483. Brendel and Bethge showed that BagNet-17 with a 17$\times$17 receptive field can achieve a similar top-5 accuracy as AlexNet~\cite{brendel2019approximating}. In recent works~\cite{zhang2020clipped,xiang2020patchguard} on adversarial patch defense, BagNet has been adopted to bound the number of corrupted features to achieve robustness. 

\noindent\textbf{Clipping.} In addition to the use of CNNs with small receptive fields, we also need a secure aggregation mechanism to ensure that a small number of corrupted features only have a limited influence on the final prediction/classification. Recall that in our provable analysis (Algorithm~\ref{alg-dpg-analysis};  Section~\ref{sec-robustness-analysis}), we need the lower bound of classification logits to reason about the worst-case objectness map output. In order to impose such a lower bound, we clip all feature values into $[0,\infty]$ such that an adversarial patch cannot decrease the values of object classes significantly. It is easy to calculate the lower bound of classification logits: we only need to zero out all features within the patch location(s) and aggregate the remaining features.

\noindent\textbf{Alternative aggregation.} We propose \framework as a general framework that is compatible with any provably robust image classification technique. To further support this claim, we implement \defense using a PatchGuard classifier with robust masking secure aggregation~\cite{xiang2020patchguard}, which achieves the best clean classification accuracy and provable robust classification accuracy on high-resolution ImageNet~\cite{deng2009imagenet} dataset. We compare the performance of defenses with clipping-base and robust-making-based secure aggregation in Table~\ref{tab-clipping-dpg}. As we can see from the table, two defenses achieve high clean performance and non-trivial provable robustness, demonstrating that \framework is compatible with different provably robust image classifiers. We note that we do not choose robust masking in our main evaluation because 1) it has a looser lower bound compared with clipping while 2) it introduces a slightly higher computation overhead. Furthermore, robust masking of PatchGuard only has limited robustness against the multiple-patch attacker. In contrast, as demonstrated in Appendix~\ref{apx-multiple-patch}, the clipping-based \framework can handle multiple patches.

\section{Discussion on Multiple Patches}\label{apx-multiple-patch}
In our main body, we focus on the one-patch threat model because building a high-performance provably robust object detector against a single-patch attacker is an unresolved and open research question. In this section, we discuss \framework's robustness against multiple patches.

\noindent \textbf{Quantitative analysis of clipping-based \framework against multiple patches.} One advantage of the clipping-based robust classifier is its robustness against multiple patches. As long as the sub-procedure $\textsc{RCH-PA}(\cdot)$ of the clipping-based robust classifier can return non-trivial bounds of classification logits, we can directly plug the sub-procedure into our Algorithm~\ref{alg-dpg-analysis} to analyze the robustness against multiple patches.

We note that despite the theoretical possibility to defend against attacks with multiple patches, its quantitative evaluation for provable robustness is extremely expensive due to the large number of all possible combinations of multiple patch locations.
Consider a 32$\times$32 patch on a 416$\times$416 image. There are 148k possible patch locations (or 1.6k feature-space locations). If we are using 2 patches of the same size, the number of all location combinations becomes higher than $10^{10}$ (or 1.4M feature-space location combinations)! 

In order to provide a proof-of-concept for defense against multiple patches, we perform an evaluation on 50 PASCAL VOC images using a subset of patch locations (1/16 of all location combinations). The results are reported in Table~\ref{tab-multiple-patch}. As shown in the table, \framework is able to defend against multiple patches. Moreover, if we compare provable robustness against one 32$\times$32 (1024 px) and two 24$\times$24 patches (1152 px), which have a similar number of pixels, we can find that using two smaller patches (two 24$\times$24 patches) is only more effective for over-patch threat model but not for far-patch and close-patch threat models. This observation leads to the following remark.

\noindent\textbf{Remark: multiple patches need to be close to each other and the victim object for a stronger malicious effect.} Unlike image classification where the classifier makes predictions based on all image pixels (or extracted features), an object detector predicts each object largely based on the pixels (or features) around the object. As a result, patches that are far away from the object only have a limited malicious effect, and this claim is supported by our evaluation results in Section~\ref{sec-eval-provable} (i.e., \framework is more effective against the far-patch threat model). Therefore, multiple patches should be close to the victim object and hence close to each other for a more effective attack. In this case, the multiple-patch threat model becomes similar to the one-patch model since patches are close to each other and can merge into one single patch. That is, we can use one single patch of a larger size to cover all perturbations in multiple small patches.
\begin{table}[t]
    \centering
    \caption{Provable robustness (CR) of \framework (using a perfect clean detector) against multiple patches (evaluated on 50 PASCAL VOC images with a subset of patch locations)}
    \resizebox{\linewidth}{!}{
    \begin{tabular}{c|c|c|c}
    \toprule
         &  far-patch & close-patch&over-patch\\
         \midrule
    one 32$\times$32 patch (1024 px)     & 27.3\%&22.4\%&8.7\%\\
    two 32$\times$32 patches (2048 px)  &27.3\% & 18.0\%&3.1\%\\ 
    two 24$\times$24 patches (1152 px) &27.3\%&18.6\%&3.1\%\\
    two 16$\times$16 patches (512 px)& 27.3\%&19.3\%&5.0\%\\
    \bottomrule
    \end{tabular}}

    \label{tab-multiple-patch}
\end{table}
\begin{figure*}[!t]
\centering
\begin{minipage}[b]{0.33\linewidth}
\includegraphics[width=\linewidth]{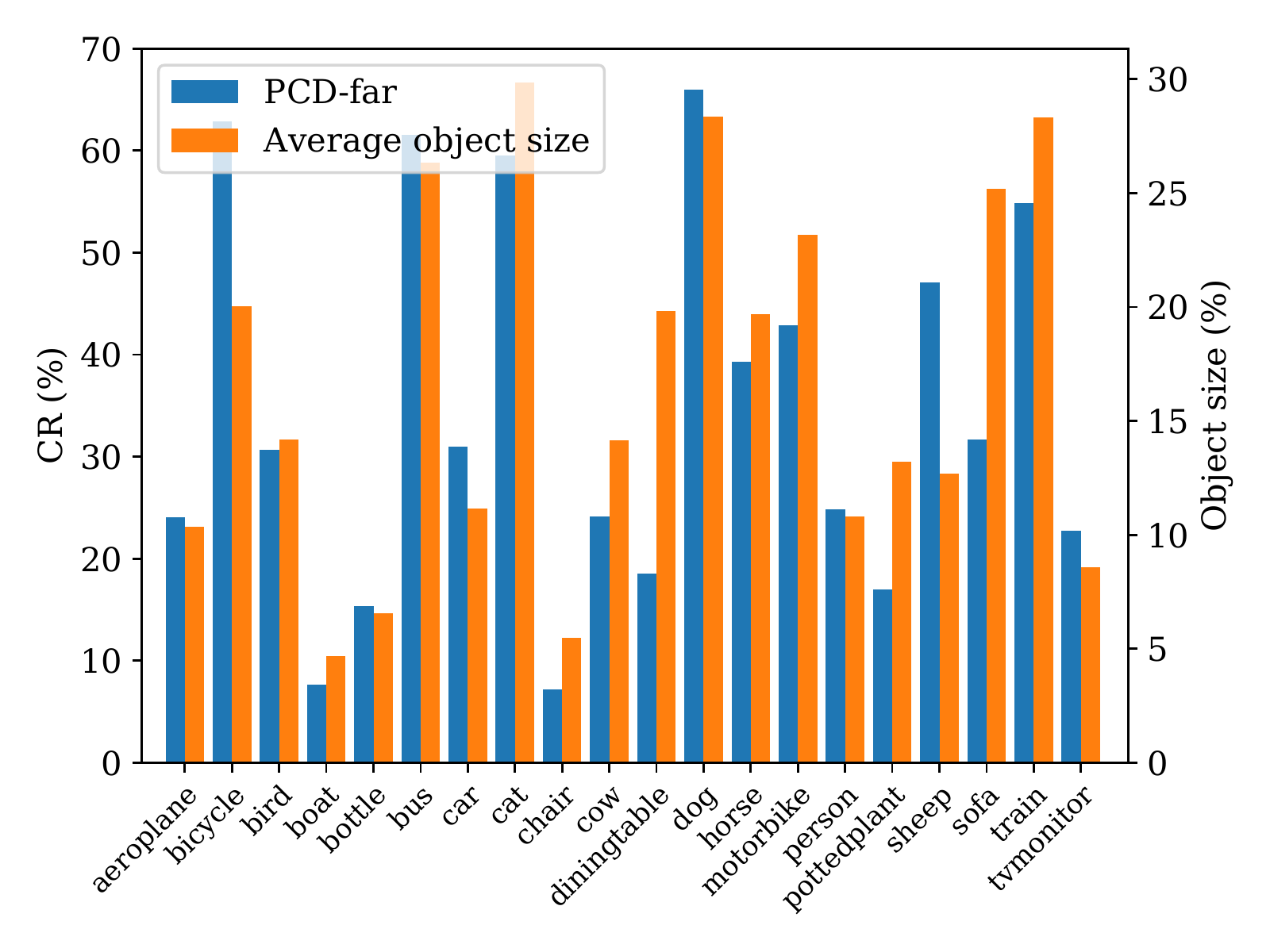}\\
\end{minipage}%
\begin{minipage}[b]{0.33\linewidth}
\includegraphics[width=\linewidth]{img/per_class_close.pdf}\\
\end{minipage}%
\begin{minipage}[b]{0.33\linewidth}
\includegraphics[width=\linewidth]{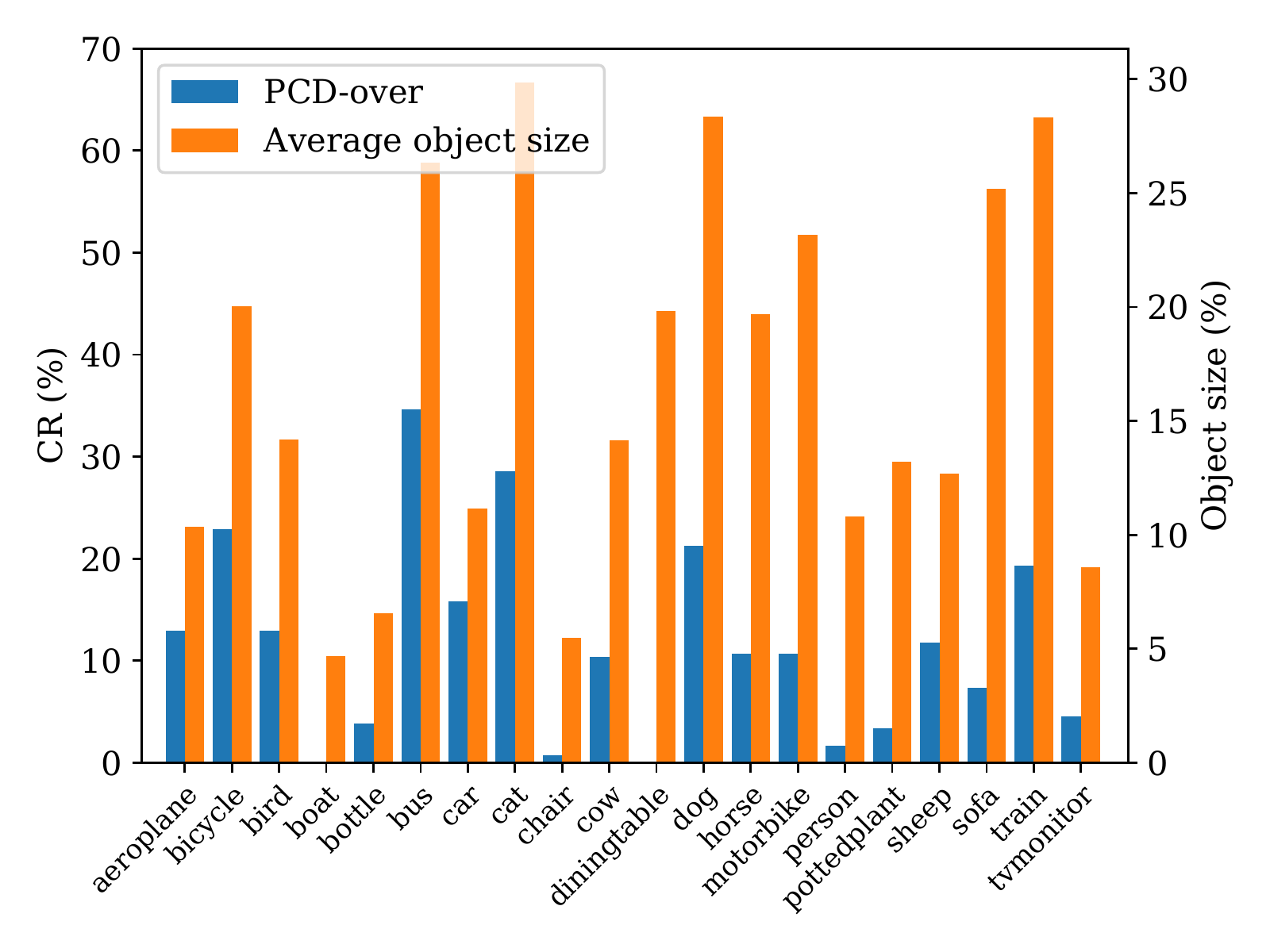}\\
\end{minipage}%
\caption{Per-class analysis of PASCAL VOC (left: far-patch; middle: close-patch; right: over-patch)}\label{fig-per-class-all}
\end{figure*}

\begin{figure}[!h]
    \centering
    \includegraphics[width=0.8\linewidth]{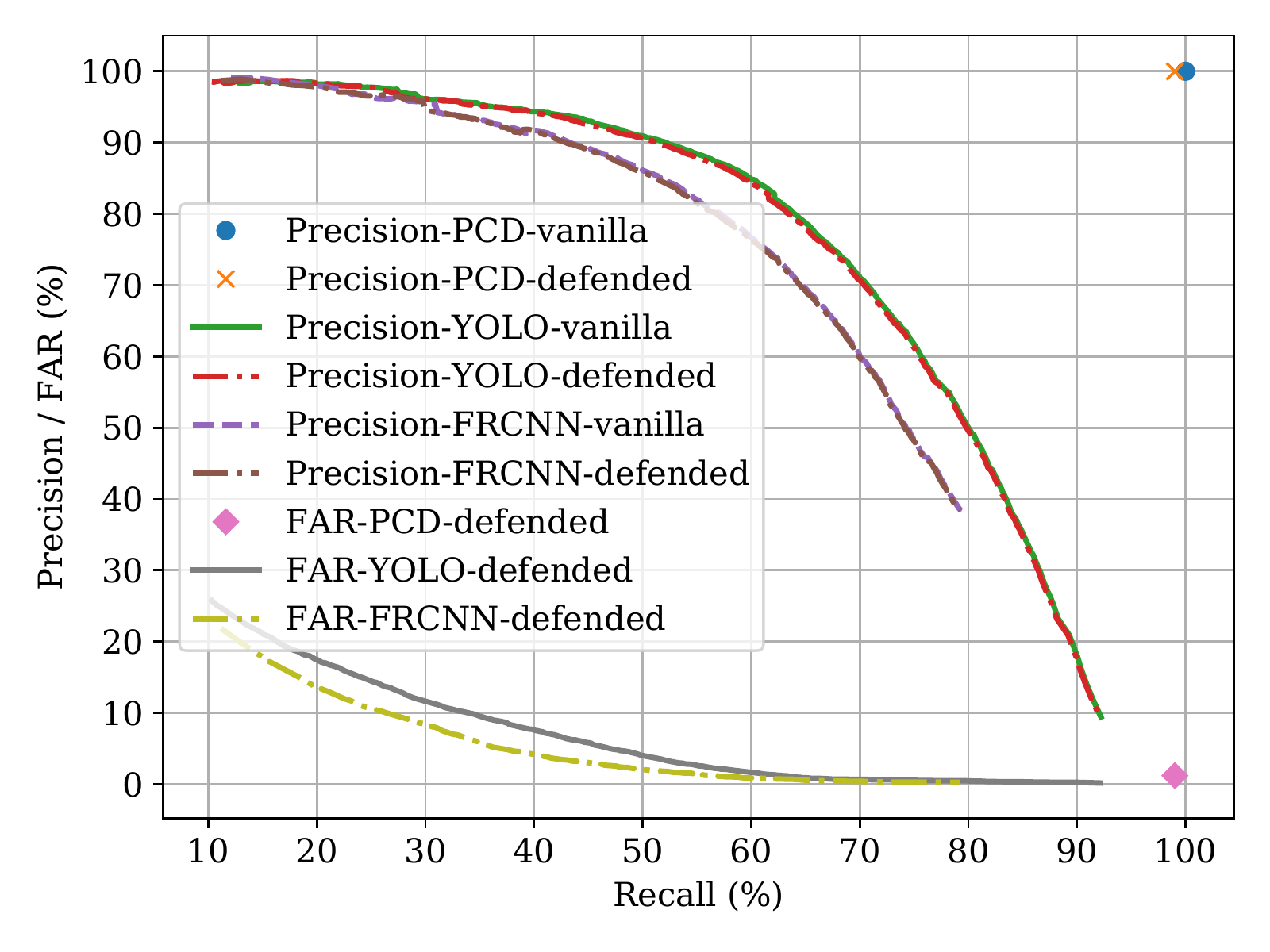}
    \caption{Clean performance of \framework on MS COCO}
    \label{fig-clean-coco}
\end{figure}

\begin{figure}[!h]
    \centering
    \includegraphics[width=0.8\linewidth]{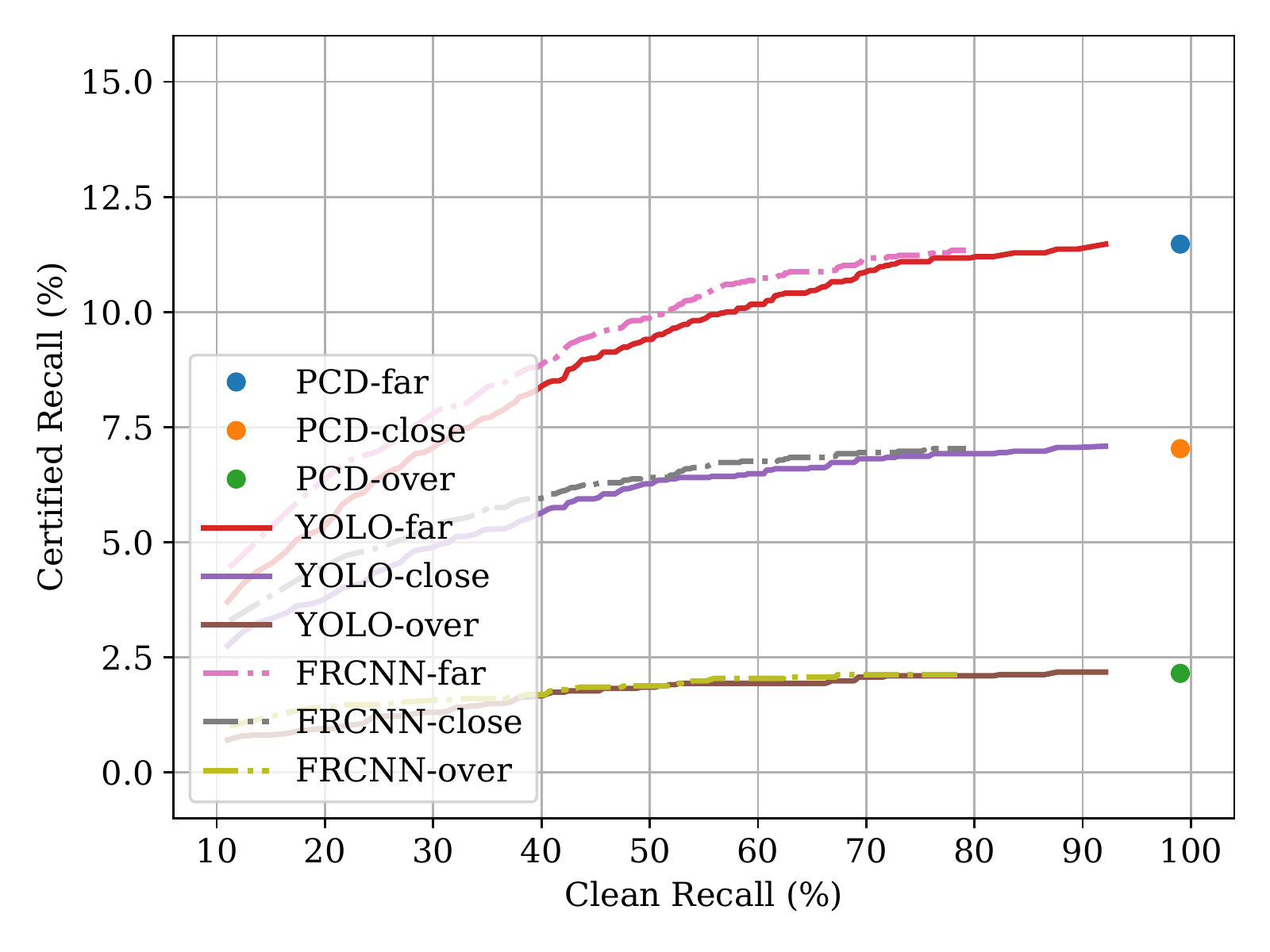}
    \caption{Provable robustness of \framework on MS COCO}
    \label{fig-pro-coco}
\end{figure}

\begin{figure}[!h]
    \centering
    \includegraphics[width=0.8\linewidth]{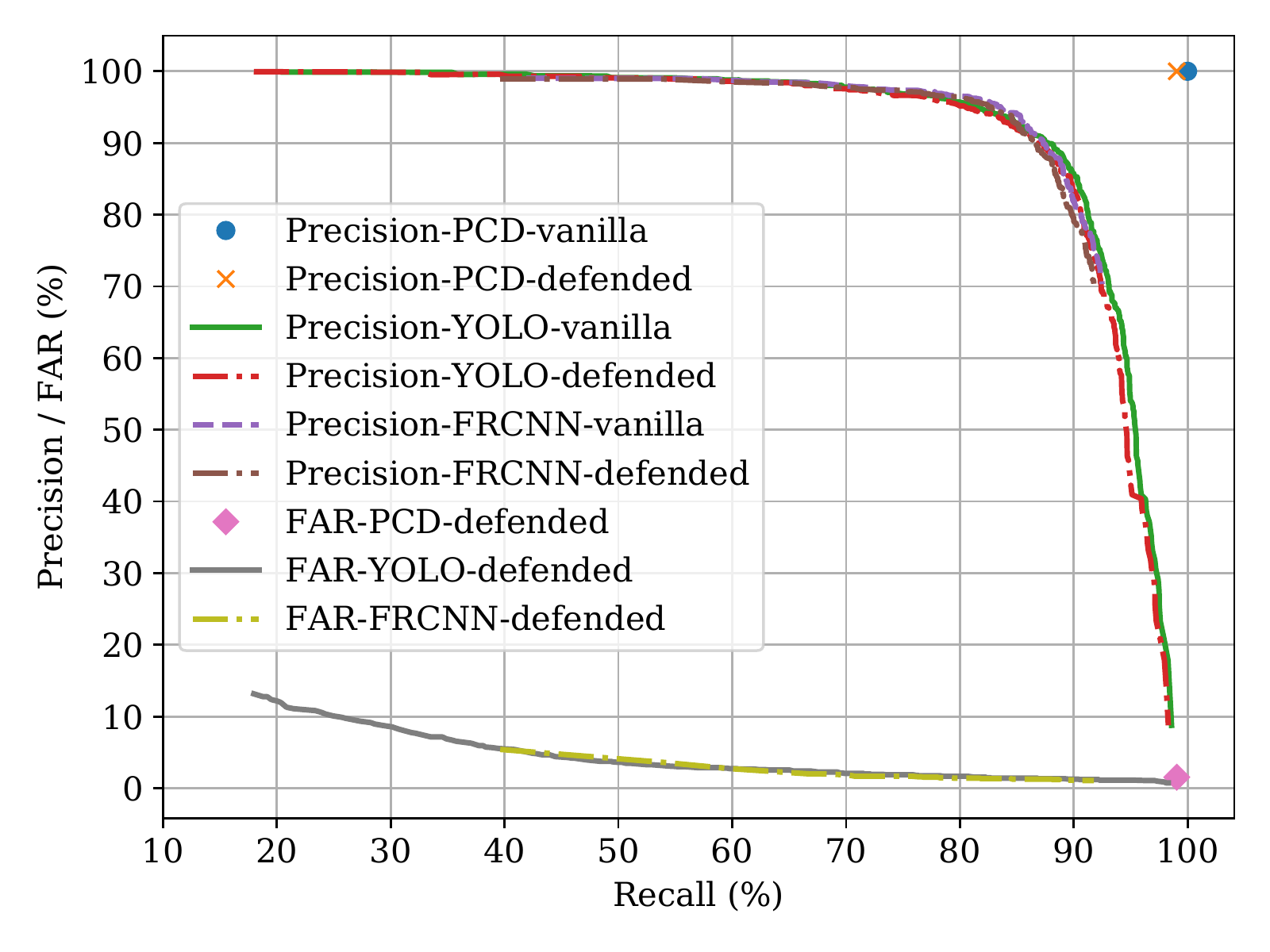}
    \caption{Clean performance of \framework on KITTI}
    \label{fig-clean-kitti}
\end{figure}

\begin{figure}[!h]
    \centering
    \includegraphics[width=0.8\linewidth]{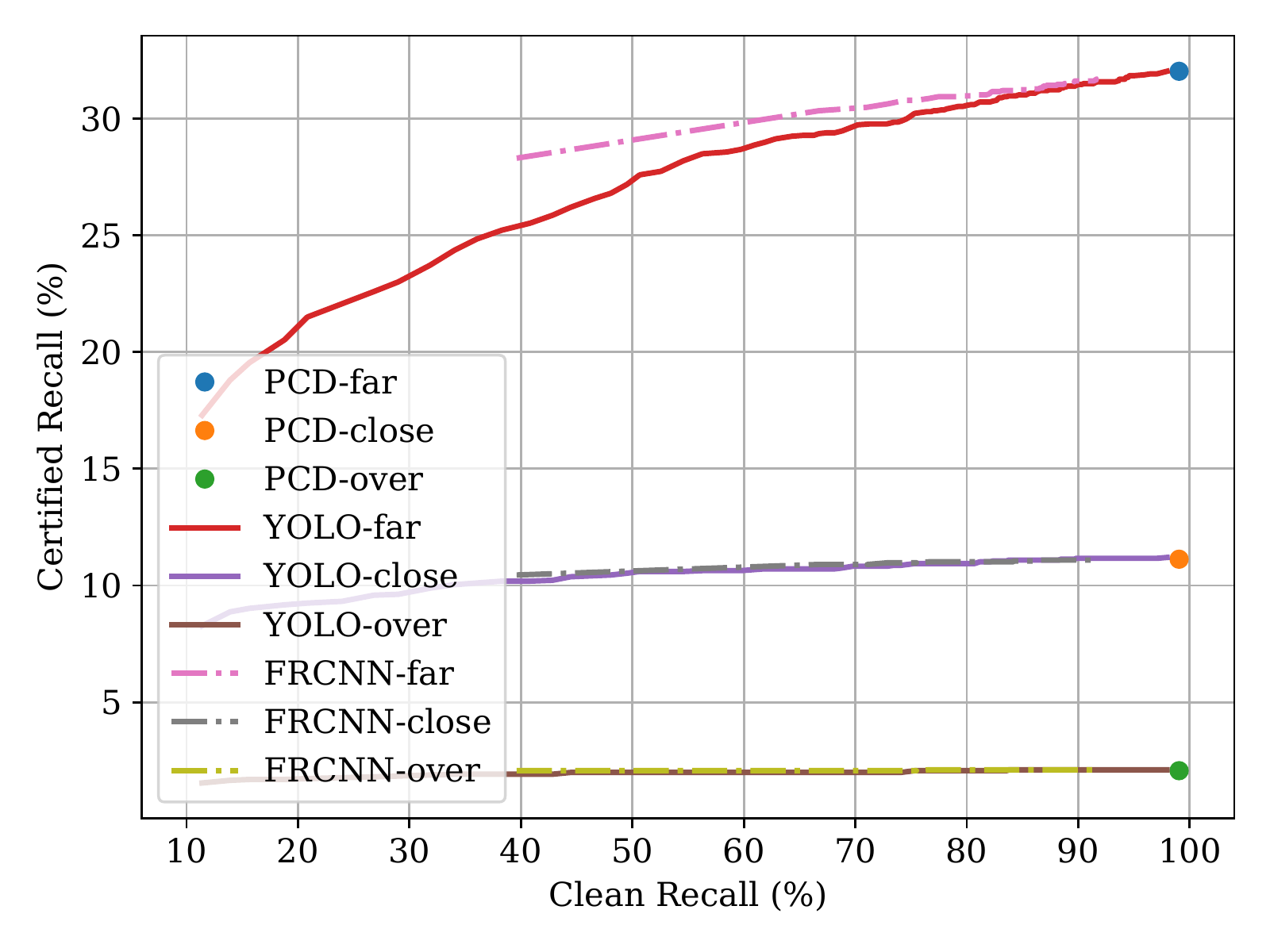}
    \caption{Provable robustness of \framework on KITTI}
    \label{fig-pro-kitti}
\end{figure}

\section{Experiment Results for Different Threat Models and Datasets}\label{apx-exp}
In this section, we include additional plots for per-class analysis as well as \framework's clean/provable performance on MS COCO and KITTI. The observation is similar to that in Section~\ref{sec-evaluation}.

\noindent\textbf{Per-class Analysis.} In Figure~\ref{fig-per-class-all}, we provide additional per-class analysis results. The observation is similar to Figure~\ref{fig-per-cls-voc-close} in Section~\ref{sec-evaluation}.

\noindent\textbf{Additional plots for MS COCO and KITTI.} We plot the clean performance and the provable robustness for MS COCO in Figure~\ref{fig-clean-coco} and Figure~\ref{fig-pro-coco}, and for KITTI in Figure~\ref{fig-clean-kitti} and Figure~\ref{fig-pro-kitti}. The observation is similar to that on PASCAL VOC (Figure~\ref{fig-clean-voc} and Figure~\ref{fig-pro-voc}).

\begin{figure*}[h]
    \centering
    \includegraphics[width=\linewidth]{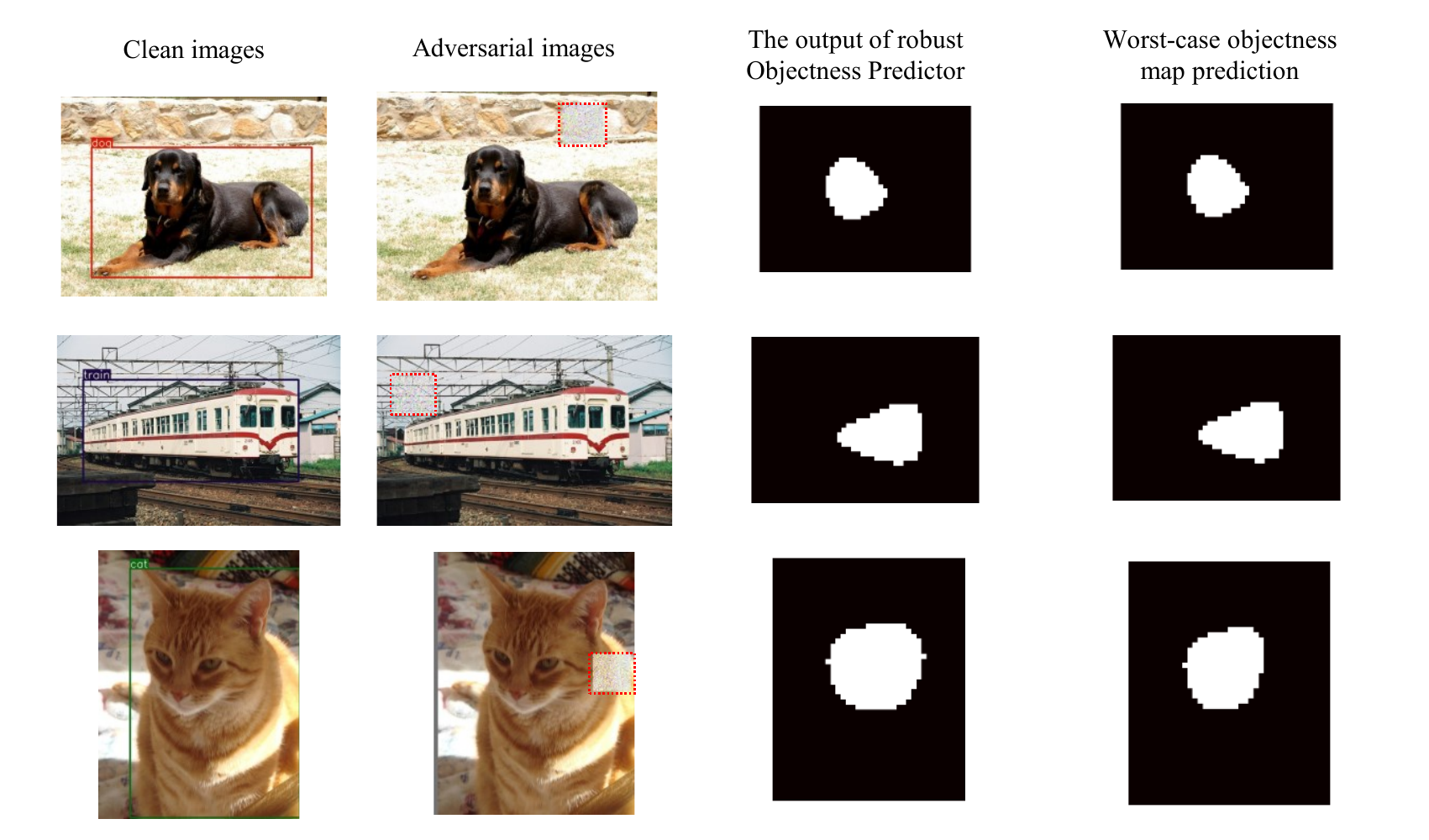}
    \caption{Visualization of \framework. \textmd{From left to right: 1) clean images -- conventional detectors correct detect all objects; 2) adversarial images with patches (marked with red dash boxes) -- conventional detectors miss all objects; 3) objectness map generated on these adversarial images -- \defense robustly predicts high objectness and eventually leads to an attack alert; 4) worst-case objectness map -- \framework can certify the provable robustness of these objects. Note that visualizations in this figure only consider one random patch location for each image, while results reported in Section~\ref{sec-evaluation} consider \textit{all possible locations and attack strategies} within the threat model.}}
    \label{fig-visu}
\end{figure*}

\section{Justification for Defense Objective}\label{apx-defense-obj}
In Section~\ref{sec-defense-formulation}, we allow \framework to \textit{only detect part of the object} or to trigger an attack alert on adversarial images. In this section, we discuss why this is a reasonable defense objective and how to extend \framework for a stronger notion of robustness.

\noindent \textbf{Partially detected bounding box.} We note that we allow the patch to be anywhere, even over the salient object. As a result, the patch likely covers a large portion of the object (visualization examples include the right part of Figure~\ref{fig-overview} and Figure~\ref{fig-patch-obj-visu}; see Appendix~\ref{apx-obj-size} for more details of object sizes and patch sizes). Therefore, it is reasonable to allow the model to output a smaller bounding box. If we consider the application scenario of  autonomous vehicles (AV), partially detecting a pedestrian or a car is already sufficient for an AV to make a correct decision.

Moreover, we can tune hyper-parameters such as binarizing threshold $T$ to increase the objectness in the output of \defense. More objectness will force the adversary to let \base predict a larger bounding box in order to reduce unexplained objectness that will otherwise lead to an attack alert. However, we note that more objectness also makes it more likely for \framework to trigger a false alert on clean images. This trade-off between robustness and clean performance should be carefully balanced.

\section{Additional Discussion on ``Free" Provable Robustness}\label{apx-why}

As shown in our evaluation, \framework achieves the first provable robustness for object detectors again patch hiding attacks at a negligible cost of clean performance. Intriguingly, we have demonstrated that we can use a module with limited clean performance (i.e., provably robust image classifier in \defense) to build a provably robust system with high clean performance (i.e., \framework). In this section, we provide additional discussion on this intriguing behavior.

One major difference between image classification and object detection is their type of error. For an image classifier, the only error is misclassification. In contrast, an object detector can have two types of errors, false-negative (FN; missing object) and false-positive (FP; predicting incorrect objects). Intriguingly, despite the difficulty to have a low FN and a low FP at the same time, it is easy to have a low FN (but with a potentially high FP) or a low FP (but with a potentially high FN). For example, if an object detector predicts all possible bounding boxes, the FN is zero (all possible boxes are retrieved, including the ground-truth box) but the FP is extremely high (most bounding boxes are incorrect). 
This intrinsic property of the object detection task allows us to achieve ``free" provable robustness.

Recall that in the clean setting, our design of prediction matching strategy (Section~\ref{sec-matcher}) enables \framework to tolerate FN of \defense (i.e., \textit{Clean Error 3}: predicts objects as background). Therefore, we can safely and easily optimize for a low FP (i.e., \textit{Clean Error 2}: predicts background as objects) to achieve a high clean performance.

\section{Pixel-space and Feature-space Windows}
\label{apx-mapping}
Recall that in Section~\ref{sec-obj-predictor}, we used a BagNet to extract a feature map for the entire image and perform robust window classification in the feature space. This design allows us to reuse the extracted feature map and reduce computational overhead. In this section, we discuss how to map the pixel-space bounding box to the feature space.

\noindent \textbf{Box mapping.} For each pixel-space box $(x_{\min},y_{\min},x_{\max},y_{\max})$, we calculate the feature-space coordinate $x_{\min}^\prime=\lfloor (x_{\min} - \texttt{r}+1)/\texttt{s} \rfloor,y_{\min}^\prime=\lfloor{(y_{\min} - \texttt{r}+1)/\texttt{s}}\rfloor,x_{\max}^\prime = \lfloor{x_{\max} /\texttt{s}}\rfloor,y_{\max}^\prime=\lfloor{y_{\max}/\texttt{s}}\rfloor$, where $\texttt{r},\texttt{s}$ are the size and stride of the receptive field size. The new feature-space coordinates indicate all features that are affected by the pixels within the pixel-space bounding box. We note that the mapping equation might be slightly different given different implementation of CNNs with small receptive fields. In our BagNet implementation, we have $\texttt{r}=33,\texttt{s}=8$.

\section{Visualization of \framework}
In this section, we give a simple visualization for \framework with YOLOv4 as \base (Figure~\ref{fig-visu}). To start with, we select three random images with larger objects and visualize the detection output of YOLOv4 in the first column. Second, we pick a random patch location on the image and perform an empirical patch hiding attack. The attack aims to optimize the pixel values within the adversarial patch to minimize the objectness confidence score of every possible bounding box prediction, which is a common strategy used in relevant literature~\cite{xu2020adversarial,wu2019making}. As shown in the second column, our patch attacks are successful, and YOLOv4 fails to detect any objects. Note that we use red dash boxes to illustrate the patch locations, and they are not the outputs of YOLOv4. Third, we feed this adversarial image to \defense, and we visualize the predicted objectness maps in the third column. As shown in the figure, although the adversarial patch makes YOLOv4 miss all objects, \defense still robustly outputs high objectness. As discussed in Section~\ref{sec-defense}, \framework will eventually issue an attack alert. Fourth, we reason about the worst-case objectness map prediction for these particular random patch locations used in the visualization, and plot the worst-case output in the fourth column. As shown in our visualization, \defense can still output high objectness in the worst case. Therefore, we can certify the robustness of \framework for this patch location.

Finally, we want to note that this appendix is merely a simple case study for an empirical patch attack at one single random location of each image. In contrast, robustness results reported in Section 5 are derived from Algorithm~\ref{alg-dpg-analysis} and Theorem~\ref{thm} holds \textit{for any possible patch attack strategy at any valid patch location.}

\end{document}